\documentclass[oneside,11pt]{article}

\usepackage{graphicx}
\usepackage{braket,amsfonts}
\usepackage{amsmath,amssymb,amsthm}
\usepackage{abstract}
\usepackage{thmtools}
\usepackage{array}
\usepackage{multirow}
\usepackage{scalerel}
\usepackage{soul}
\usepackage[margin=1.2in]{geometry}
\usepackage[caption=false]{subfig}
\usepackage{amsopn}
\DeclareMathOperator*{\argmin}{arg\,min}
\usepackage{algorithm}
\usepackage{algpseudocode}
\usepackage{graphicx,epstopdf}
\usepackage[hidelinks]{hyperref}
\newcommand{\email}[1]{\protect\href{mailto:#1}{#1}}
\RequirePackage[capitalize,nameinlink]{cleveref}[0.19]

\newcommand{\crefassum}[1]{Universal Assumptions}

\newcommand{\crefobjdef}[1]{Problem Assumption}

\crefname{section}{section}{sections}
\crefname{subsection}{subsection}{subsections}
\Crefname{section}{Section}{Sections}
\Crefname{subsection}{Subsection}{Subsections}
\Crefname{figure}{Figure}{Figures}

\crefformat{equation}{\textup{#2(#1)#3}}
\crefrangeformat{equation}{\textup{#3(#1)#4--#5(#2)#6}}
\crefmultiformat{equation}{\textup{#2(#1)#3}}{ and \textup{#2(#1)#3}}
{, \textup{#2(#1)#3}}{, and \textup{#2(#1)#3}}
\crefrangemultiformat{equation}{\textup{#3(#1)#4--#5(#2)#6}}%
{ and \textup{#3(#1)#4--#5(#2)#6}}{, \textup{#3(#1)#4--#5(#2)#6}}{, and \textup{#3(#1)#4--#5(#2)#6}}

\Crefformat{equation}{#2Equation~\textup{(#1)}#3}
\Crefrangeformat{equation}{Equations~\textup{#3(#1)#4--#5(#2)#6}}
\Crefmultiformat{equation}{Equations~\textup{#2(#1)#3}}{ and \textup{#2(#1)#3}}
{, \textup{#2(#1)#3}}{, and \textup{#2(#1)#3}}
\Crefrangemultiformat{equation}{Equations~\textup{#3(#1)#4--#5(#2)#6}}%
{ and \textup{#3(#1)#4--#5(#2)#6}}{, \textup{#3(#1)#4--#5(#2)#6}}{, and \textup{#3(#1)#4--#5(#2)#6}}

\crefdefaultlabelformat{#2\textup{#1}#3}
\newtheorem{definition}{Definition}
\newtheorem{theorem}{Theorem}
\newtheorem{Assumption}{Assumption}
\newtheorem{Lemma}[theorem]{Lemma}
\newtheorem{Proposition}[theorem]{Proposition}
\newtheorem{Corollary}[theorem]{Corollary}

\newcommand{\Real}{\mathbb{R}}

\newcommand{\srh}{\mathrm{SR}_\varepsilon}
\newcommand{\ssrh}{\mathrm{signed}\text{-}\mathrm{SR}_\varepsilon}

\newcommand{\csr}{\mathrm{SR}}
\newcommand{\rn}{\mathrm{RN}}
\newcommand{\float}{\mathbb{F}}
\newcommand{\fl}{\mathrm{fl}}
\newcommand{\sign}{\mathrm{sign}}

\newcommand{\expt}{\mathrm{E}}

\newcommand{\bh}{\mathbf h}
\newcommand{\bx}{\mathbf x}
\newcommand{\by}{\mathbf y}
\newcommand{\bz}{\mathbf z}
\newcommand{\one}{\mathbf 1}
\newcommand{\bdelta}{\boldsymbol{\delta}}
\newcommand{\bsigma}{\boldsymbol{\sigma}}
\newcommand{\wh}{\widehat}
\newcommand{\ds}{\displaystyle}
\DeclareMathOperator*{\argmax}{arg\,max}

\title{On the influence of stochastic roundoff errors and their bias on the convergence of the gradient descent method with low-precision floating-point computation}

\author{Lu~Xia\thanks{Department of Mathematics and Computer Science, Eindhoven University of Technology, PO Box 513, 5600 MB Eindhoven, The Netherlands (\email{\{l.xia1, m.e.hochstenbach, b.koren\}@tue.nl}, ).}\and Stefano Massei\thanks{Department of Mathematics, Università di Pisa, 56127 Pisa, Italy (\email{stefano.massei@unipi.it}).}\and Michiel E. Hochstenbach\footnotemark[1]\and Barry~Koren\footnotemark[1] }

\date{}

\begin{document}
\maketitle
\begin{abstract}
When implementing the gradient descent method in low precision, the employment of stochastic rounding schemes helps to prevent stagnation of convergence caused by the vanishing gradient effect. Unbiased stochastic rounding yields zero bias by preserving small updates with probabilities proportional to their relative magnitudes. This study provides a theoretical explanation for the stagnation of the gradient descent method in low-precision computation. Additionally, we propose two new stochastic rounding schemes that trade the zero bias property with a larger probability to preserve small gradients. Our methods yield a constant rounding bias that, on average, lies in a descent direction. For convex problems, we prove that the proposed rounding methods typically have a beneficial effect on the convergence rate of gradient descent. We validate our theoretical analysis by comparing the performances of various rounding schemes when optimizing a multinomial logistic regression model and when training a simple neural network with an 8-bit floating-point format. 
\end{abstract}

\section{Introduction}
\label{sec:intro}

Low-precision computations attract increasing attention as they allow to drastically minimize the use of computational resources \cite{chung2018serving,hickmann2020intel,jouppi2020domain}. Adopting a lower precision generally introduces larger roundoff errors. The magnification of roundoff errors may cause divergence of numerical methods. Hence, it is crucial to analyze the error propagation in algorithmic procedures and to investigate the effect of different rounding schemes \cite{higham2002accuracy}. 

In addition to classical deterministic rounding strategies, such as round down, round up, and round to the nearest ($\rn$), techniques that incorporate randomized procedures have been proposed. According to Croci et al.~\cite{croci2022stochastic}, in 1949, Huskey and Hartree \cite{huskey1949precision} introduced an unbiased stochastic rounding scheme in their pioneering work, that we call \emph{stochastic rounding} ($\csr$), to reduce the accumulated round-off errors in solving ordinary differential equations. In the last decade, applications and analysis of stochastic rounding techniques have emerged \cite{connolly2021stochastic,croci2022stochastic,davies2018loihi,mikaitis2021stochastic}. Our work is inspired by the recent paper of Gupta et al. \cite{gupta2015deep}, where it is empirically shown that, with 16-bit fixed-point representation, the training of neural networks (NNs) stagnates with $\rn$ while $\csr$ preserves a very similar performance to single-precision computation. This has motivated further investigations of the use of 
$\csr$ in training NNs with low-precision computations \cite{na2017chip,ortiz2018low,wang2018training}. Besides in machine learning, $\csr$ has been recently applied in climate modeling \cite{paxton2021climate} and in solving partial differential equations with low precision \cite{croci2020effects,hopkins2019stochastic}, and its implementation in hardware is also growing \cite{davies2018loihi,mikaitis2021stochastic,su2020towards}.

Many tasks related to machine learning, e.g., training NNs, linear regression and logistic regression, are carried out by means of the gradient descent method (GD). The latter is also widely employed in many other areas; see, e.g., \cite{liu2015fuzzy,petres2007path}. The convergence of GD in exact arithmetic is well understood; see, e.g., \cite{lee2016gradient,nesterov2003introductory,zou2020gradient}. In general, the convergence analysis of stochastic or inexact GD only addresses the errors in evaluating the gradient function \cite{bertsekas2000gradient,schmidt2011convergence}. The convergence of GD in training quantized NN is studied with respect to $\csr$, only considering the errors in storing the updating parameters in low precision \cite{li2017training}. A systematic roundoff error analysis that accounts for either deterministic or stochastic rounding errors throughout the whole updating procedure of GD is lacking. Additionally, the theoretical explanation for the stagnation of GD in low-precision computation is insufficient.

In this paper, we theoretically explore the role of stochastic rounding methods in preventing the stagnation of GD. We analyze the influence of floating-point roundoff errors on the convergence of GD with fixed stepsize for convex problems. Our analysis considers three types of roundoff errors obtained by the GD iteration: the errors obtained in evaluating the gradient, in computing the multiplication of the rounded gradient with the stepsize, and in determining the subtraction. 
We analyze these rounding errors for two scenarios. For Scenario 1, we consider the case that GD does not suffer from stagnation and is evaluated using stochastic rounding methods. In Scenario 2, we consider a special case in which GD suffers from stagnation with $\rn$ using limited-precision computation.  
To force the rounding bias in a descent direction, we propose two new stochastic rounding methods that we call \emph{$\varepsilon$-biased stochastic rounding} ($\srh$) and \emph{signed $\varepsilon$-biased stochastic rounding} ($\ssrh$). The proposed biased stochastic rounding methods have been proven to eliminate the stagnation of GD and provide a significantly faster convergence than $\csr$ in low-precision floating-point computation.
The novelties of this work are as follows
\begin{enumerate}
    \item The utilization of $\csr$ has been for the first-time theoretically proven to help prevent stagnation of GD in low-precision computation;
    \item the rounding bias is for the first-time applied to accelerate the convergence of GD;
    \item two novel biased stochastic rounding methods are proposed.
\end{enumerate}

The outcomes of our convergence analysis concern two aspects of GD: monotonicity and convergence rate. A summary of the convergence analysis with respect to different scenarios and steps is given in \cref{tab:summarylayout}. 
 We validate these theoretical results with experiments on quadratic functions, training a multinomial logistic regression model (MLR) and a two-layer NN (a non-convex problem). The results confirm that, with the same precision, both $\srh$ and $\ssrh$  generally provide faster convergence than $\rn$ and $\csr$.
\begin{table}[htb!]
{\begin{center}
\footnotesize
\caption{Summary of the main theoretical results.}\label{tab:summarylayout}
\begin{tabular}{lll}\cline{1-3}
\rule{0pt}{2.3ex}Convergence analysis & Rounding scheme & Reference\\\cline{1-3}
\rule{0pt}{2.3ex}Monotonicity & General rounding & \cref{thm:monotone}\\
Convergence rate & General rounding &\cref{lem:boundforfloating_point}\\
Convergence rate &$\csr$ \cref{eq:gd_floatp1} and $\csr$ \cref{eq:gd_floatp2} &\cref{thm:conv-exp_csr}\\ 
Convergence rate &$\srh$ \cref{eq:gd_floatp1} and $\csr$ \cref{eq:gd_floatp2} &\cref{coro:conv-exp_srh}\\
Monotonicity & $\csr$ \cref{eq:gd_floatp1} and $\csr$ \cref{eq:gd_underflowSRB} & \cref{the:ucsr}\\
Monotonicity & $\csr$ \cref{eq:gd_floatp1} and $\ssrh$ \cref{eq:gd_underflowSRB} & \cref{thm:modifiedsrh}\\
\cline{1-3}
\end{tabular}
\end{center}
}
\end{table}

The work is organized as follows. In \cref{sec:sec2}, we recall the basic properties of floating-point arithmetic and $\csr$, and we introduce $\srh$ and $\ssrh$. The source of rounding errors when implementing GD with floating-point representation is analyzed in \cref{sec:GDwithfp}. In \cref{sec:convergence}, we study the influence of rounding bias on the convergence of GD for convex problems for the three scenarios, i.e., for deterministic roundoff errors, for stochastic roundoff errors, and for a special case when GD stagnates with $\rn$. Then, we validate our theoretical analysis with numerical experiments in \cref{sec:simulation}. Conclusions are drawn in \cref{sec:conclusion}.
\section{Number representation system and rounding schemes}\label{sec:sec2}
We start this section by recalling some basic properties of the floating-point arithmetic and with the definition of $\csr$. Then, we introduce the $\srh$ and $\ssrh$ schemes.
\subsection{Floating-point representation}\label{sec:flp} 
A floating-point system \cite{8766229} $\mathbb F\subset \Real$ is a proper subset of real numbers. A floating-point number $\wh x\in \mathbb F$ can be represented by radix $2$ (binary representation), significand precision $s$, and exponent $e$ \cite[Sec.~2.1]{higham2002accuracy}, as $\wh x = \pm\, \mu \cdot 2^{e-s},$ where $\mu, e$, and $s$ are integers satisfying $\mu \in [0, 2^s-1]$, and $e\in [e_{\min}, e_{\max}]$. We call \emph{rounding} any map that associates with $x\in\Real\backslash\float$ a certain $\wh x\in\float$. The unit roundoff $u:=2^{-s}$ is the maximum relative error caused by approximating a real number $x \in \Real\backslash\float^{\star}$ by $\wh x \in \float^{\star}$ using $\rn$, where $\float^{\star}=\float\backslash\{\text{subnormal numbers}\}$ \cite[Sec.~2.1]{higham2002accuracy}. 

A technical standard is the IEEE Standard for Floating-Point Arithmetic (IEEE 754) \cite{8766229}. According to IEEE 754, there are five basic formats for binary computation, i.e., {\sf binary16} (half precision), {\sf binary32} (single precision), {\sf binary64} (double precision), and two others. For a detailed description of floating-point number formats, see \cite[Sec.~3]{8766229} and \cite[Sec.~2.1]{higham2002accuracy}. Since we are primarily concerned with floating-point arithmetic, we refer to the roundoff error as representing the relative rounding error. In the numerical experiments, we employ {\sf binary32}, {\sf bfloat16}, and {\sf binary8}. The format {\sf bfloat16} has 8 exponent bits and supports an 8-bit precision \cite{connolly2021stochastic} and {\sf binary8} has the same number format as the {\sf E5M2} format on NVIDIA H100 tensor core \cite{navidiah100}. A summary of the parameters for the number formats is given in \cref{table:numberformat}. For the convergence analysis, we focus on the default rounding mode used in IEEE 754 floating-point operations, i.e., $\rn$, and the stochastic rounding methods that we outline in this section.
\begin{table}[ht!]
{\footnotesize
\caption{Summary of the parameters of the number formats. $u$ is the unit roundoff, $x_{\min}$ is the smallest normalized positive number, and $x_{\max}$ is the largest finite number.}\label{table:numberformat}
\begin{center}
\begin{tabular}{llll}
 \cline{1-4} \rule{0pt}{2.3ex}%
 Format & $u$ & $x_{\min}$& $x_{\max}$\\ \cline{1-4}\rule{0pt}{2.5ex}%
{\sf binary8}  & $2^{-3}$& $6.10\cdot10^{-5}$& $5.73\cdot10^4$ \\
{\sf bfloat16} & $2^{-8}$ &$1.18\cdot10^{-38}$ &$3.39\cdot10^{38}$ \\
{\sf binary16} & $2^{-11}$& $6.10 \cdot 10^{-5}$ & $6.55 \cdot 10^4$\\
{\sf binary32} & $2^{-24}$ &$1.18 \cdot10^{-38}$ & $3.40 \cdot 10^{38}$\\
{\sf binary64} & $2^{-53}$& $2.22\cdot10^{-308}$ & $1.80 \cdot 10^{308}$\\\cline{1-4}
\end{tabular}
\end{center}
}\end{table}
\subsection{Stochastic rounding} \label{sec:rounding}
We denote by $\fl(\cdot)$ a general rounding operator that maps $x\in \Real$ into $\fl(x)\in \float$. When a specific rounding scheme is applied, $\fl(\cdot)$ will be replaced by the corresponding rounding operator. The most natural choice for a rounding operation is to opt for one of the two floating-point numbers that are adjacent to $x$. More precisely, rounding schemes choose $\fl(x)\in\{\lfloor x \rfloor,\lceil x \rceil\}$ where $\lfloor x \rfloor:=\max\{y\in \float: y \le x\}$ and $\lceil x \rceil:=\min\{y\in \float: y \ge x\}$. A stochastic rounding scheme chooses whether $\fl(x)=\lfloor x\rfloor$ or $\fl(x)=\lceil x\rceil$, according to a certain $x$-dependent probability. We write, for $x\ne 0$,
$\sigma(x):= \fl(x)-x$, and $\delta(x):=(\fl(x)-x)/x$,
the absolute and relative errors, respectively. An appropriate superscript will be added when the latter quantities refer to a specific rounding scheme. Now let us review the $\csr$ scheme.
\begin{definition}(Cf., e.g., \cite{connolly2021stochastic})
\label{def:stochasticrounding}
For $x\in \Real$, the rounded value $\csr(x)$ is defined as 
\begin{align*}
\csr(x)=\begin{cases}
\lfloor x \rfloor, \quad &\text{with probability}~p_{0}(x):=1-\tfrac{x-\lfloor x \rfloor}{\lceil x \rceil-\lfloor x \rfloor},\\
\lceil x \rceil, \quad &\text{with probability}~1-p_{0}(x)=\tfrac{x-\lfloor x \rfloor}{\lceil x \rceil-\lfloor x \rfloor}.
\end{cases}
\end{align*}
\end{definition} 
$\csr$ has rounding probability depending on its input $x$, in such a way that $0$ rounding bias is achieved, i.e., the expectation of this stochastic process satisfies $\expt\,[\,\sigma^{\scaleto{\csr}{4pt}}(x)\,]=0$, for all $x$. 
To preserve more information when dealing with small gradients, we propose to set a lower bound $\varepsilon<1$ to the probability of rounding away from $0$. More formally, we introduce a new stochastic rounding scheme, $\srh$, as follows.
\begin{definition} \label{def:srh}
Given $\varepsilon\in(0,1)$, we define the following functions 
\begin{align}\label{eq:epsilon}
&\eta(x,\varepsilon) := 1-\frac{x-\lfloor x \rfloor}{\lceil x \rceil-\lfloor x \rfloor}-\sign(x)\,\varepsilon,~ \quad
\quad \varphi(y)=\begin{cases}
0,&y\le 0,\\
y,&0\le y\le 1,\\
1,&y\ge 1.
\end{cases} \end{align}
Then define $p_\varepsilon(x):=\varphi(\eta(x, \varepsilon))$ and
\begin{align}
\srh(x)=\begin{cases}
\lfloor x\rfloor, \quad &\text{with probability}~p_\varepsilon(x),\\
\lceil x\rceil, \quad &\text{with probability}~1-p_\varepsilon(x).
\end{cases}
\label{eq:modified_csr}
\end{align} 
\end{definition}
For a fixed $x\in \Real$, $\srh(x)$ is a discrete random variable with sample space $\{\lfloor x\rfloor,\lceil x\rceil\}$. With a direct computation we get the following expression for the expected absolute rounding error: 
\begin{equation} \label{eq:csrh-experror}
\expt\,[\,\sigma^{\scaleto{\srh}{5.5pt}}(x)\,]=\begin{cases}
 \lfloor x\rfloor-x,&\eta(x,\varepsilon)>1,\\
 \sign(x) \, \varepsilon \, (\lceil x \rceil-\lfloor x \rfloor),&0\le \eta(x,\varepsilon) \le 1,\\
 \lceil x\rceil-x,&\eta(x,\varepsilon)<0.
\end{cases}
\end{equation}
\begin{figure*}[htb!]
\centering
\subfloat[$x>0$]{\label{fig:roundingdisa} \includegraphics[width=0.38\textwidth]{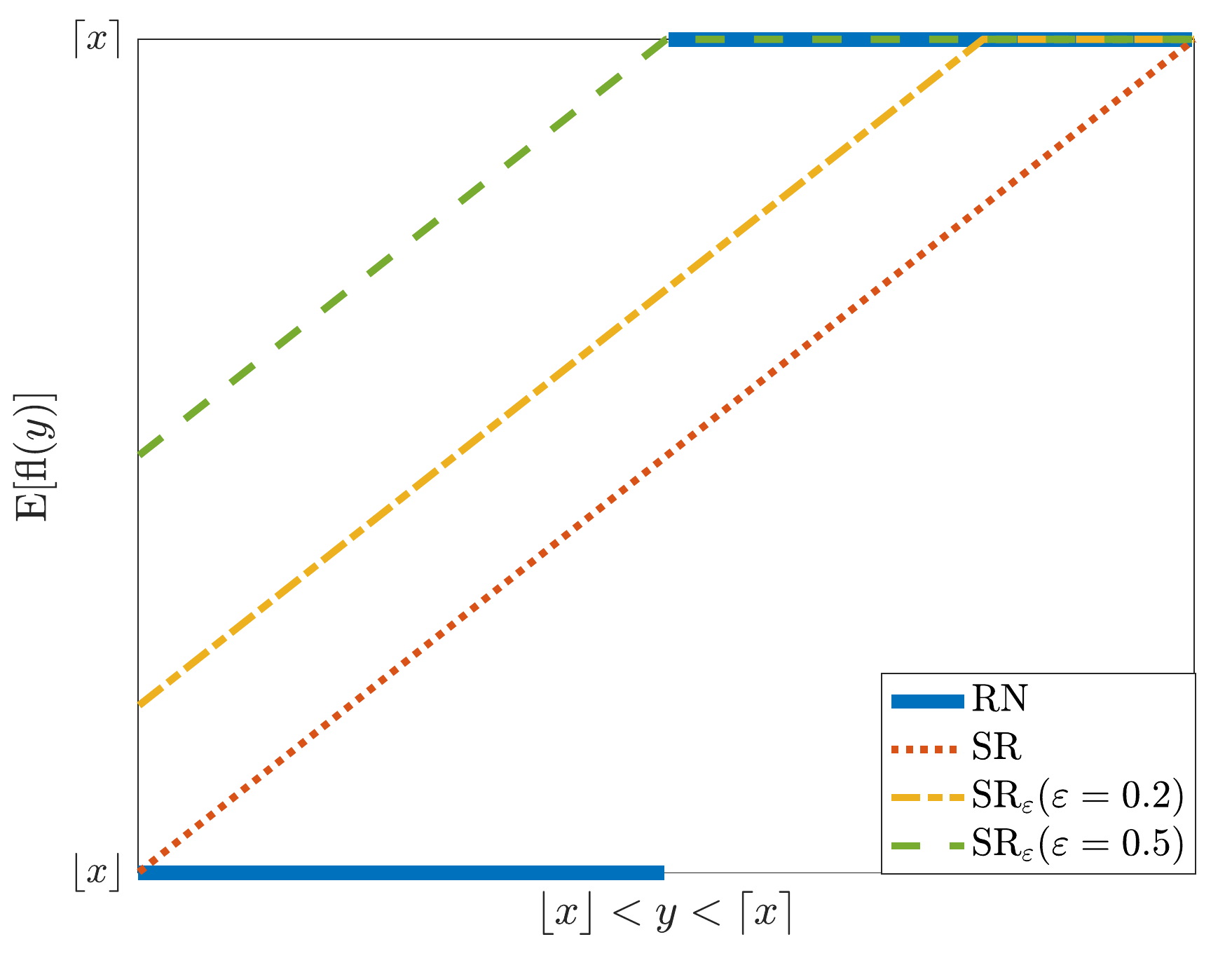}} \quad\quad
\subfloat[$x<0$ ]{\label{fig:roundingdisb} \includegraphics[width=0.38\textwidth]{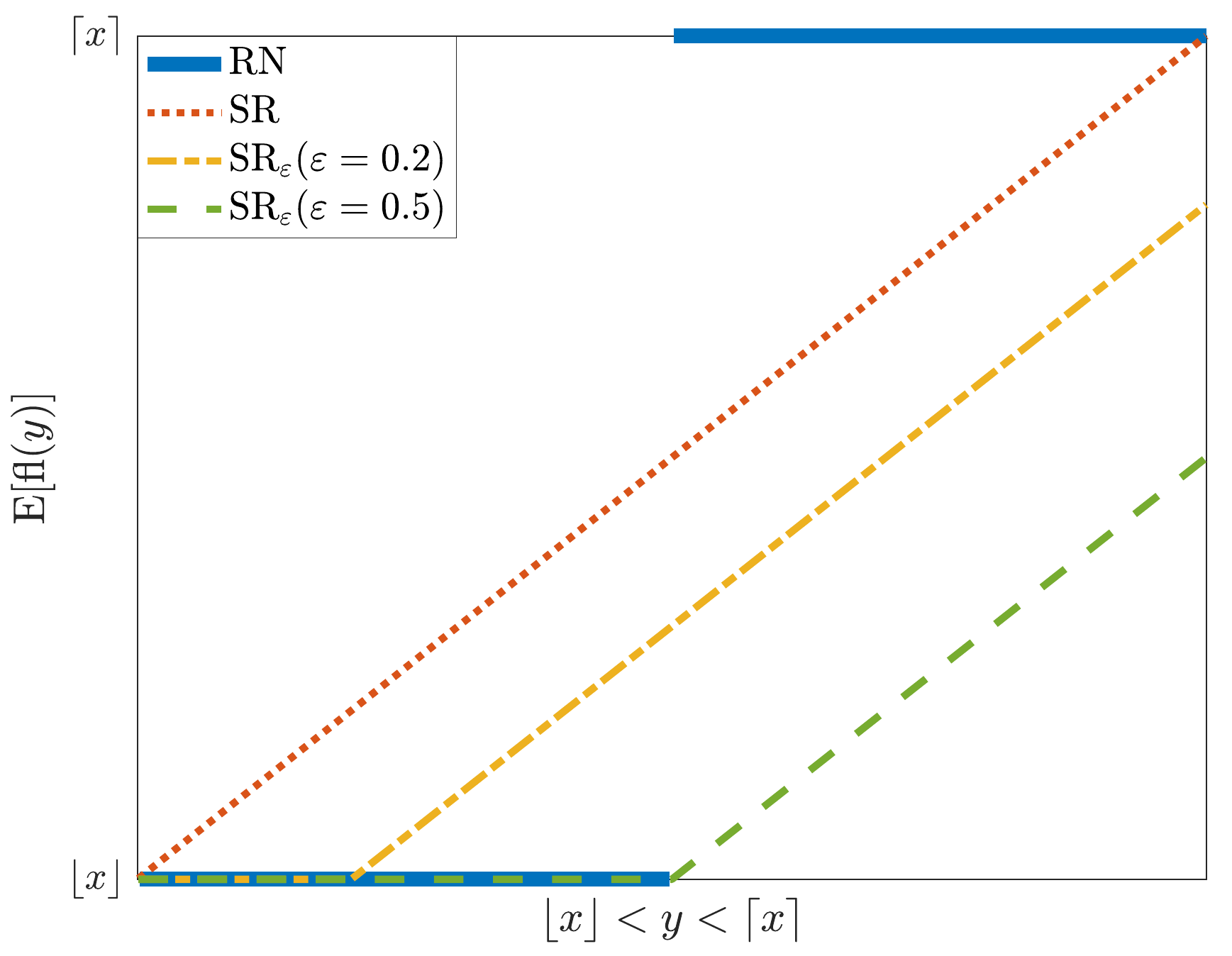}}
\caption{Comparison of $\expt\,[\,\mathrm{fl}(y)\,]$ for $y\in(\lfloor x \rfloor, \lceil x \rceil)$ (for a fixed $x$) using different rounding schemes for $x>0$ (a) and $x<0$ (b).}\label{fig:roundingdis}
\end{figure*}

\cref{fig:roundingdis} plots the value of $\expt\,[\,\fl(y)\,]$ for $y \in (\lfloor x \rfloor,\lceil x \rceil)$, using the rounding schemes introduced so far. It can be seen that when $x>0$, $\srh$ combines stochastic rounding and ceiling, while we have a combination of stochastic rounding and flooring when $x<0$. When $\varepsilon\le0.5$, the deterministic rounding part of $\srh$ behaves as $\rn$ (see \cref{fig:roundingdis} where $\varepsilon\leq 0.5$). In \cref{def:srh}, $\srh$ leads to a rounding bias with the same sign as its input. By introducing an additional variable $v\in \Real$ and with a minor modification of the function $\eta$ in \cref{def:srh}, we obtain a new stochastic rounding method, which we name $\ssrh$, with rounding bias in the opposite sign of $v$. In the context of GD with low-precision representation, we will use this method to get a constant rounding bias in a descent direction by substituting the corresponding entries of the gradient vector for $v$. 
In \cref{sec:directionofroundingerrors}, we will show how $\ssrh$ may be beneficial for implementing GD. The $\ssrh$ method is defined as follows.
\begin{definition} \label{def:signedsrh}
 Let $\sign(v)$ be a desired sign of rounding bias, $\varphi$ as introduced in \cref{def:srh}, and
$
\wh{\eta}(x,\varepsilon,v) := 1-\frac{x-\lfloor x \rfloor}{\lceil x \rceil-\lfloor x \rfloor}+\sign(v)\,\varepsilon.
$
We define $\wh{p}_{\varepsilon}(x):= \varphi(\wh{\eta}(x,\varepsilon,v))$ and (cf.~\cref{eq:modified_csr})
\[
\ssrh(x)=\begin{cases}
\lfloor x\rfloor, \quad &\text{with probability}~ \wh{p}_{\varepsilon}(x),\\
\lceil x\rceil, \quad &\text{with probability}~1- \wh{p}_{\varepsilon}(x).
\end{cases}
\]
\end{definition}
By direct computation we get the following expression for the expected absolute rounding error for $\ssrh$ (cf.~\cref{eq:csrh-experror}): 
\begin{equation} \label{eq:ssrh-experror}
\expt\,[\,\sigma^{\scaleto{\ssrh}{6pt}}(x)\,]=\begin{cases}
 \lfloor x\rfloor-x,&\wh{\eta}(x,\varepsilon,v)>1,\\
 \sign(-v) \, \varepsilon \, (\lceil x \rceil-\lfloor x \rfloor),&0\le \wh{\eta}(x,\varepsilon,v) \le 1,\\
 \lceil x\rceil-x,&\wh{\eta}(x,\varepsilon,v)<0.
\end{cases}
\end{equation}
It can be seen that when $0\le \wh{\eta}(x,\varepsilon,v) \le 1$, the expected absolute rounding error always has a sign opposite to that of $v$. Therefore, when applying $\ssrh$ to implement GD, one can achieve a rounding bias in a descent direction by replacing $v$ with the components of the gradient vector.  
 
\subsection{Standard models of arithmetic operations} Standard models of floating-point operations ($\mathrm{op}$) are based on $\rn$. For $\mathrm{op} \in \{+,-, *,/,\sqrt{}\}$, they satisfy \cite[Sec.~2.2]{higham2002accuracy}
\begin{align}\label{eq:flmodel}
\rn(x)&=x\,(1+\delta),\quad \text{and} \quad
\rn(x~\mathrm{op}~y)=(x~\mathrm{op}~y)(1+\delta),\!\!\!\!&\vert \delta \vert \le u.
\end{align}

For $\csr$, by assuming that the elementary operations and the square root are stochastically rounded to the exact ones, \cref{eq:flmodel} holds when replacing $u$ by $2u$; see \cite[(2.4)]{connolly2021stochastic}. For $\srh$ we can identify the following two cases due to the role of $\varepsilon$ in the rounding probability:
\begin{align}
\srh(x)=x\,(1+\delta),\qquad
\begin{cases}
\vert \delta \vert \le 2u, & 0\le \eta(x,\varepsilon) \le 1,\\
0\le \delta \le 2\,\varepsilon\,u, & \text{otherwise}.\\
\end{cases}
\label{eq:fltoper_srh}
\end{align}
For $\srh(x~\mathrm{op}~y)$ the same bounds hold as in \cref{eq:fltoper_srh}, by assuming that $\mathrm{op} \in \{+,-, *,/,\sqrt{}\}$ is the stochastically rounded exact one for $\csr$. Finally, the same bounds which hold for $\srh$ also apply to $\ssrh$. Under the same assumption, we may achieve an upper bound for the expected relative error of $\srh$.

\begin{Lemma}\label{lem:exp_deltasrh}
Under the assumption that the elementary operations and the square root are stochastically rounded to the exact ones, we have that the expectation of the corresponding relative error satisfies $0\le\expt\,[\delta^{\scaleto{\srh}{5.5pt}}(x)\,]\leq 2\,\varepsilon\,u$, for all nonzero $x\in\mathbb R$ and $0<\varepsilon<1$. 
\end{Lemma}
\begin{proof}
Observing \cref{eq:epsilon}, we have that $\eta(x,\varepsilon)>1$ happens only for $x<0$, which implies that $\frac{x-\lfloor x\rfloor }{\lceil x\rceil-\lfloor x\rfloor}\le\varepsilon$. The fact that $\eta(x,\varepsilon)<0$ happens only for $x>0$ gives that $\frac{\lceil x\rceil-x }{\lceil x\rceil-\lfloor x\rfloor}\le\varepsilon$. On the basis of \cref{eq:csrh-experror}, we get $\expt\,[\,\delta^{\scaleto{\srh}{6pt}}(x)\,]\leq \frac{\varepsilon \,(\lceil x\rceil-\lfloor x\rfloor)}{x}\leq 2\,\varepsilon\,u.$
\end{proof}

\subsection{Implementation of stochastic rounding}
In our implementation, all three stochastic rounding schemes, $\csr$, $\srh$ and $\ssrh$, are obtained by slightly modifying the \texttt{roundit} function in the \texttt{chop} MATLAB function \cite{higham2019simulating}. The \texttt{chop} function is developed to round the elements of a matrix to a lower precision floating-point arithmetic with certain rounding methods defined in the subfunction \texttt{roundit}. The input variable of \texttt{roundit} is already scaled with the desired precision, and therefore it is not necessary to consider the scaling process in the \texttt{roundit} function. Note that \texttt{chop} is implemented with double-precision computation ({\sf binary 64}), therefore only a lower precision than {\sf binary 64} can be evaluated. A pseudo code is given in \cref{alg:roundspdpfloat}, where $\csr$ can be evaluated by setting $\varepsilon=0$ and $\srh$ can be evaluated by setting $v=x$ and $0<\varepsilon<1$.
\begin{algorithm}[ht!]
	\caption{Pseudocode of $\ssrh$ using the \texttt{roundit} function in \cite{higham2019simulating}.}\label{alg:roundspdpfloat}
	\begin{algorithmic}[1]		
		\State Input: $x$ and $\varepsilon$ and $v$ (cf.~\cref{def:signedsrh}).
		\State Compute $y = \vert x\vert$ and $\sigma = y - \lfloor y \rfloor $. 
            \State The probability of rounding down is $\wh{p}_{\varepsilon}(y)=\varphi(1-(y-\lfloor y \rfloor)+\sign(v)\,\sign(x)\,\varepsilon)$ .
		\State Generate a random number $\xi$ from the random number generator.
            \State Then the rounded value can be computed as $\mathrm{fl}(y)= \begin{cases}
		\lfloor y \rfloor   \quad  &\text{if $\xi\le \wh{p}_{\varepsilon}(y)$},\\
		\lceil y\rceil  \quad    &\text{else}.
		\end{cases}$
		\State Output: $y = \sign(x)\cdot \mathrm{fl}(y)$.
	\end{algorithmic}
\end{algorithm}

\section{Gradient descent in floating-point arithmetic} \label{sec:GDwithfp} We recall the GD algorithm for minimizing a differentiable function $f:\Real^n\to \Real$. For a  fixed \emph{stepsize} $t>0$, the method iteratively updates in the opposite direction of the gradient with the rule
\begin{equation}
\bx^{(k+1)}=\bx^{(k)}-t \, \nabla f(\bx^{(k)}).
\label{eq:gd}
\end{equation}

\subsection{Source of rounding errors} \label{sec:sourceoferrors}
We denote by $\wh \bx^{(k)}$ the sequence generated by GD in finite precision. When implementing GD with floating-point numbers, there are three sources of rounding errors in the evaluation of \cref{eq:gd} that we have to take into account: the accumulated absolute rounding error $\bsigma_1$ arising from computing the gradient (see \cref{eq:gd_floatp0}), the roundoff error $\bdelta_2$ coming from the multiplication with the stepsize $t$ (see \cref{eq:gd_floatp1}), and $\bdelta_3$ caused by the final subtraction (see \cref{eq:gd_floatp2}). Note that we use an accumulated absolute rounding error for the computation of the gradient since in general the evaluation of the gradient may not be backward stable; for inner products and matrix-vector products see \cite[Secs.~3.1 and 3.5]{higham2002accuracy}.
We split the GD iteration into the following three steps:
\begin{subequations} \label{eq:gd_floatp}
\begin{align}
\wh{\nabla f(\wh{\bx}^{(k)})} &=\nabla f(\wh{\bx}^{(k)})+\bsigma_1^{(k)}, \label{eq:gd_floatp0}\\
 \bz^{(k+1)} &=\wh{\bx}^{(k)}-t\,\wh{\nabla f(\wh{\bx}^{(k)}}) \circ(\one+\bdelta_2^{(k)}), \label{eq:gd_floatp1}\\
\wh{\bx}^{(k+1)} &=\bz^{(k+1)} \circ(\one+\bdelta_3^{(k)}),\label{eq:gd_floatp2}\end{align}
\end{subequations}
 where $\one\in \Real^n$ is the vector of all ones and $\circ$ indicates the Hadamard product. The magnitudes of the entries of $ \bdelta_2^{(k)}, \bdelta_{3}^{(k)}$ are bounded by either $u$ or $2u$; see \cref{eq:flmodel} and \cref{eq:fltoper_srh}. We denote the corresponding absolute errors by $\bsigma_m^{(k)}$ and $\bh_m^{(k)}=\one+\bdelta_{m}^{(k)}$, for $m=2,3$. Finally, we remark that all the entries in $h_{m,i}^{(k)}$ are positive, since the rounding schemes mentioned in \cref{sec:rounding} do not change the sign of the output of a single arithmetic operation.

Although in floating-point arithmetic, the relative error is bounded by $u$ or $2u$ for a single operation, when a series of operations is implemented high relative accuracy may not be guaranteed. For instance, high
relative accuracy is not guaranteed for evaluating the inner product $\bx^T\by$ when $\vert \bx^T\by\vert \ll \vert \bx\vert^T \vert \by \vert$ (see, e.g., \cite[p. 63]{higham2002accuracy}). Therefore, for the gradient evaluation, we use the following bound including both absolute and relative errors, i.e., the entries of $\bsigma_1^{(k)}$ satisfy the bound 
\begin{align} \label{eq:boundsigma1}
 \vert\sigma_{1,i}^{(k)} \vert \le c\,u\,(\vert \nabla f(\wh{\bx}^{(k)})_i\vert +1),
\end{align}
where $c\ge 0$ is a non-negative constant dependent on $\nabla f$, which can be obtained analytically for a given $\nabla f$. For instance, for a quadratic function $\tfrac{1}{2} \, \bx^T\!A\bx$, if $A$ is a diagonal matrix, $c=2$. When $A$ is a full matrix and the iterates $\bx$ stay in the compact set $\{\by\in\mathbb R^n:\, \Vert \by\Vert_\infty\le M\}$, we can take $c= \frac{2nu\, \Vert A\Vert_{\infty} \,M}{1-2nu}$ (cf.~\cite[Sec.~3.5]{higham2002accuracy}). Note that one can also choose different values of $c$ for absolute error and relative error, i.e., $\vert\sigma_{1,i}^{(k)} \vert \le c_1\,u\,\vert \nabla f(\wh{\bx}^{(k)})_i\vert +c_2\,u$. For the sake of simplicity, we choose the same $c$ in \cref{eq:boundsigma1}.
\subsection{Stagnation and non-stagnation of GD with RN} \label{sec:stag} 
For every run of GD with limited precision and $\rn$, after a certain number of iteration steps, stagnation may happen due to rounding, i.e., $\wh{\bx}^{(k+1)}=\rn\big(\wh{\bx}^{(k)}-\rn(\,t\,\rn(\nabla f(\wh{\bx}^{(k)}))\,)\big)=\wh{\bx}^{(k)}$. Usually, this phenomenon happens earlier, with respect to the number of iteration steps, with low-precision computations. 
Let us have a closer look at this situation in the spirit of \cite[Thm.~2.2]{higham2002accuracy}. We denote $z_i^{(k+1)}=\wh{x}_i^{(k)}-\rn(t\,\rn(\nabla f(\wh{\bx}^{(k)})_i)\,)=\mu_i^{(k)}2^{e_i^{(k)}-s}$, where $\mu_i^{(k)}\in[2^{s-1},2^s)$ and $e_i^{(k)}\in\mathbb{N}$. We define
\[
\tau_k:= \max_{i=1,\dots,n} 2^{-e_i^{(k)}}\rn(t\,\rn(\nabla f(\wh{\bx}^{(k)})_i)\,)
\]
and $i_k:=\argmax_{i=1,\dots,n} 2^{-e_i^{(k)}}\rn(t\,\rn(\nabla f(\wh{\bx}^{(k)})_i)\,)$
as the maximum value and the corresponding index, respectively.
If $\tau_k \le \tfrac{1}{2}u$ and the least significant bit of $\wh{x}_{i_k}^{(k)}$ equals $0$, then GD stagnates with $\rn$ and only converges to a neighborhood of the optimal point. 

As an illustrative example, in \cref{fig:underflow} we show the trajectory of $\wh x^{(k)}$ when minimizing $f(x)=(x-1024)^2$ using GD with RN and {\sf binary8}. Although $1024$ can be represented exactly using $\rn$, when $k\ge 8$ GD stagnates as $\tau_k=0.046$; see \cref{fig:ratio}. Note that the stagnation may not always occur to GD with low-precision computation. When solving multi-dimensional optimization problems, GD may only stagnate along some coordinates of $\bx$. 
\begin{figure}[ht!]
\centering
\subfloat[$x$ trajectory ]{\label{fig:underflow}\includegraphics[width=0.42\textwidth]{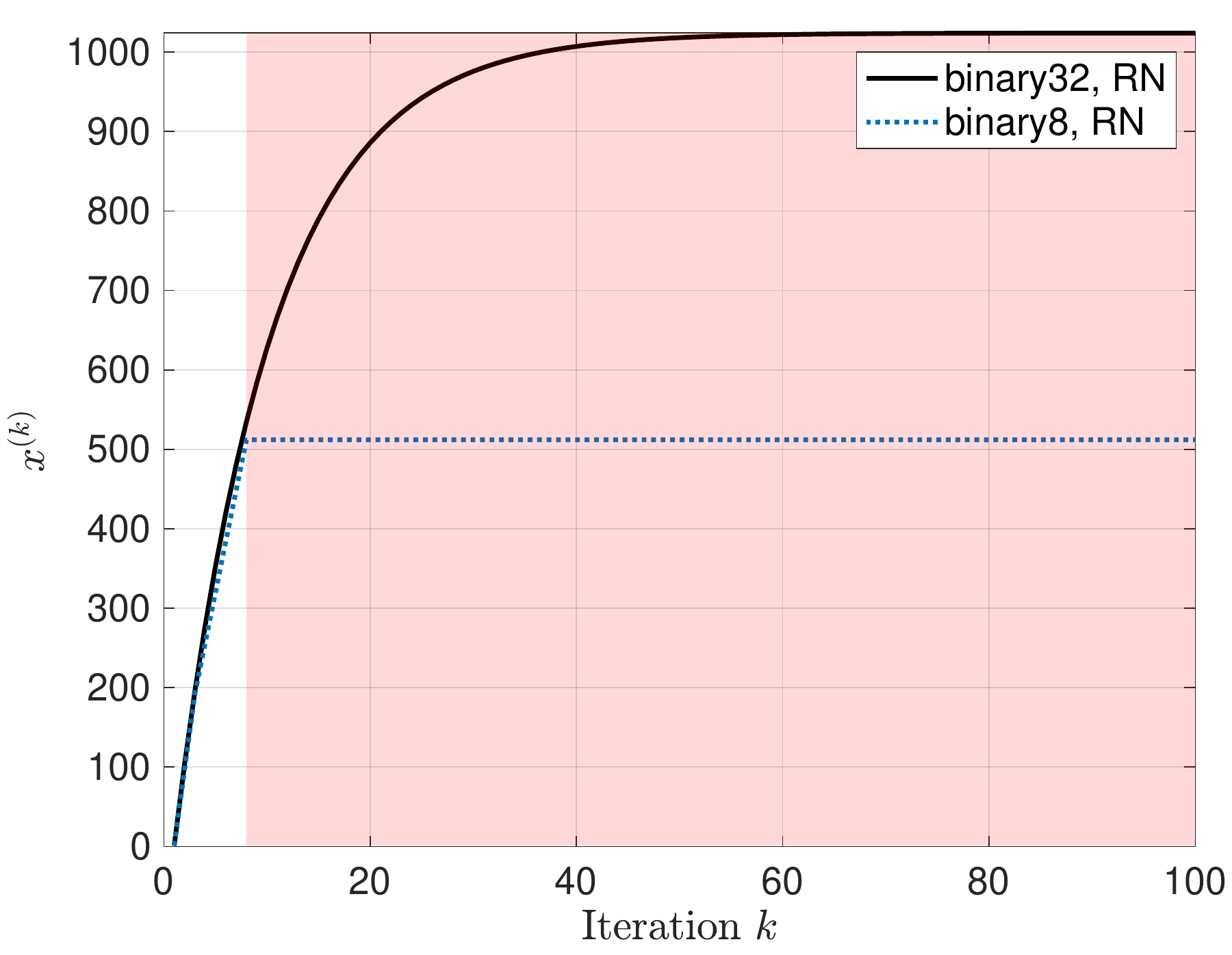}}\quad\quad
\subfloat[ratio]{\label{fig:ratio}\includegraphics[width=0.42\textwidth]{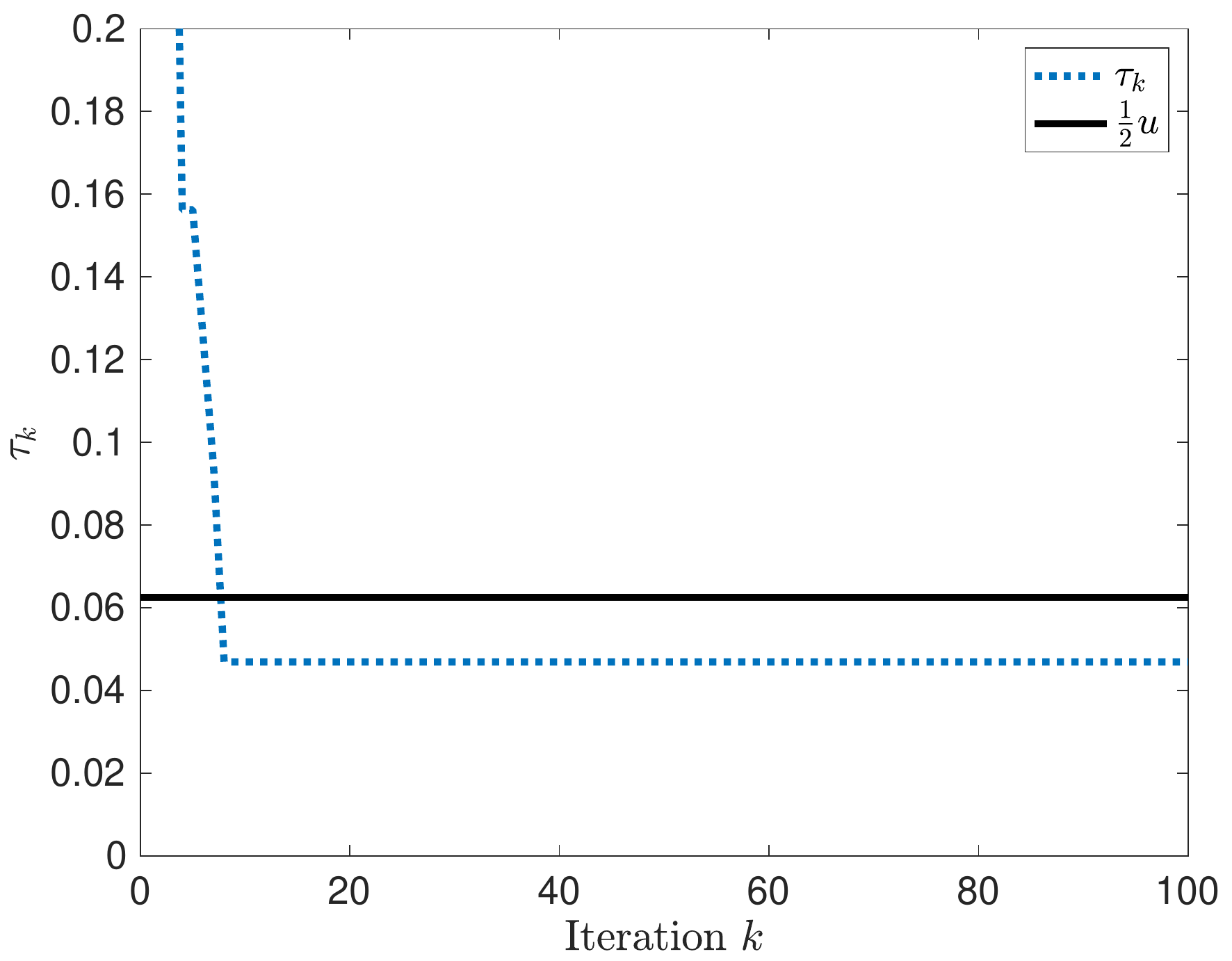}}
\caption{Minimizing $f(x)=(x-1024)^2$ using GD with {\sf binary8} ($u=2^{-3}$) and $\rn$, where the red area indicates where stagnation occurs.}
\end{figure}

Given the phenomena depicted in \cref{fig:underflow}, we split the analysis of the convergence of GD into the two scenarios mentioned in the Introduction. In Scenario 1, we focus on stochastic rounding methods that do not suffer from stagnation of GD, i.e., $\tau_k>\tfrac{1}{2}u$ in all iteration steps and in Scenario 2 we consider a special case when GD stagnates with $\rn$ ($\tau_k \le \tfrac{1}{2}u$ for a certain $k$). However, computing $\tau_k$ is impractical as it requires the exact (or accurate) value of the gradient update. Therefore, we propose the following conditions as an interpretation of $\tau_k$ for the two scenarios. To facilitate the analysis, we introduce the functions that return the successor and the predecessor of a given floating-point number $\wh x\in\float$:
\begin{align} \label{eq:flbound}
\mathrm{su}(\wh{x})=\min\{\wh{y}>\wh{x}\ \big\vert \ \wh{y}\in \float\} \quad \text{and}\quad \mathrm{pr}(\wh{x})=\max\{\wh{y}<\wh{x}\ \big\vert\  \wh{y}\in \float\}.
\end{align}
Note that $\mathrm{su}(\cdot)$ and $\mathrm{pr}(\cdot)$ differ from the ceiling and floor operations in view of the strict inequalities in \cref{eq:flbound}.  

\textbf{Condition of scenario 1} (no stagnation). $\rn$ rounds $\wh{x}_i^{(k)} -t\,\vert(\nabla f(\wh{\bx}^{(k)})_i+\sigma_{1,i}^{(k)})\,h_{2,i}^{(k)}\vert$ to either $\mathrm{su}(\wh{x}_i^{(k)})$ or $\mathrm{pr}(\wh{x}_i^{(k)})$, for $i=1,\dots, n$, i.e., we have that 
\begin{align}\label{eq:conditon_1}
 \bigg\vert\frac{t\,(\nabla f(\wh{\bx}^{(k)})_i+\sigma_{1,i}^{(k)})\,h_{2,i}^{(k)}}{\mathrm{su}(x_i^{(k)})-x_i^{(k)}}\bigg\vert > \tfrac{1}{2} \quad \text{or} \quad \bigg\vert\frac{t\,(\nabla f(\wh{\bx}^{(k)})_i+\sigma_{1,i}^{(k)})\,h_{2,i}^{(k)}}{x_i^{(k)}-\mathrm{pr}(x_i^{(k)})}\bigg\vert > \tfrac{1}{2}.\end{align}
 
 \textbf{Condition of scenario 2} (stagnation). $\rn$ rounds $\wh{x}_i^{(k)} -t\,\vert(\nabla f(\wh{\bx}^{(k)})_i+\sigma_{1,i}^{(k)})\,h_{2,i}^{(k)}\vert$ to $\wh{x}_i^{(k)}$, for $i=1,\dots, n$, i.e., we have that 
\begin{align}\label{eq:conditon_2}
 \bigg\vert\frac{t\,(\nabla f(\wh{\bx}^{(k)})_i+\sigma_{1,i}^{(k)})\,h_{2,i}^{(k)}}{\mathrm{su}(x_i^{(k)})-x_i^{(k)}}\bigg\vert \le \tfrac{1}{2} \quad \text{or} \quad \bigg\vert\frac{t\,(\nabla f(\wh{\bx}^{(k)})_i+\sigma_{1,i}^{(k)})\,h_{2,i}^{(k)}}{x_i^{(k)}-\mathrm{pr}(x_i^{(k)})}\bigg\vert \le \tfrac{1}{2}.\end{align}
Note that there may be a scenario that some of the components of $\bx^{(k)}$ satisfy condition \cref{eq:conditon_1} and some satisfy \cref{eq:conditon_2}; the convergence analysis will be the combination of two scenarios and maintain the monotonicity under the same condition as the above scenarios. For the simplicity of analysis, we only consider the above two scenarios. 
\section{Convergence analysis of GD for convex problems} \label{sec:convergence}
Throughout this section we consider the unconstrained optimization problem 
$\argmin_{\bx\in \mathbb{R}^n} f(\bx)$,
of which the objective function is assumed to satisfy the following condition and we denote by $\mathbf x^*\in\mathbb R^{n}$ the minimizer of $f$.
\begin{Assumption}
\label{model:convexfun}
The function $f$ is convex and its gradient $\nabla f: \mathbb{R}^n \to \mathbb{R}^n$ is Lipschitz continuous with constant $L>0$ and the error in evaluating the gradient satisfies \cref{eq:boundsigma1}. 
\end{Assumption}
This means that, denoting the Euclidean norm by $\Vert\cdot\Vert$, for all $\mathbf x, \by\in \Real^n$, $f$ satisfies:
\begin{align}
 f(\mathbf x)+\nabla f(\mathbf x)^T(\by-\mathbf x)\,\le \, f(\by)\,\le \, f(\mathbf x)+\nabla f(\mathbf x)^T(\by-\mathbf x)+\tfrac{1}{2}L\,\Vert \by-\mathbf x\Vert^2.\label{eq:lc_ineq}
\end{align}
The proof of \cref{eq:lc_ineq} can be found in, e.g., \cite[Thm.~2.1.5]{nesterov2003introductory}. On the basis of  \cref{eq:boundsigma1}, we have
\begin{align}
\expt\,[\,\vert\,\sigma_{1,i}^{(k)}\,\nabla f(\wh{\bx}^{(k)})_i\,\vert\,]&\le (\expt\,[\,\nabla f(\wh{\bx}^{(k)})_i^2\,]+\expt\,[\,\vert\nabla f(\wh{\bx}^{(k)})_i \vert \,])\,c\,u .\label{eq:sigma_1_c2}
\end{align}
Condition \cref{eq:sigma_1_c2} is applicable to general rounding methods and optimization problems. Additionally, we also investigate the impact of stochastic rounding errors on the convergence of GD for a special case when $\expt\,[\,\sigma_{1,i}^{(k)}\ \big\vert \ \nabla f(\wh{\bx}^{(k)})_i\,]=0$ holds, i.e.,
\begin{align}
\expt\,[\,\sigma_{1,i}^{(k)}\,\nabla f(\wh{\bx}^{(k)})_i\,]&=0.\label{eq:sigma_1_c1}
\end{align}
It is easy to check that condition \cref{eq:sigma_1_c1} is satisfied when using $\csr$ to evaluate $\sigma_{1,i}^{(k)}$ for a type of optimization problems, e.g., quadratic functions $(\bx-\bx^*)^TA(\bx-\bx^*)$ or when the gradient function is evaluated exactly.

 First let us recall the convergence rate of GD for exact arithmetic. When exact arithmetic is employed, the convergence rate of GD is at least sublinear with respect to the number of iteration steps, as the following result shows.
\begin{theorem}(\cite[Thm.~2.1.14, Cor.~2.1.2]{nesterov2003introductory})\label{theorem:convergencerate} 
Under \cref{model:convexfun}, the $k$th iterate of the gradient descent method with a fixed stepsize $t\le \frac{1}{L}$ satisfies the following inequality: 
\begin{equation}
f(\bx^{(k)})-f(\bx^*) \le \frac{2L}{4+L\,t\,k} \, \Vert \bx^{(0)}-\bx^*\Vert^2.
\label{eq:convergencerate_general}
\end{equation}
\end{theorem}
\cref{theorem:convergencerate} ensures that, in exact arithmetic, GD asymptotically converges to the optimum. However, when implementing GD in floating-point arithmetic, the method may only converge to some level of accuracy depending on the rounding precision. 

Before stating the results for the two scenarios, let us discuss a bound on $u$, i.e., $u\le \frac{a}{c\,+\,4a\,+\,4}$ with $0<a<1$, that guarantees the convergence when the number of significant bits is limited. We show that a smaller $a$ yields a smaller lower bound on the norm of the gradient, while a larger $a$ causes a larger lower bound.
\begin{Proposition} \label{prop:boundonu}
Let the objective function f satisfy \cref{model:convexfun}. Assume that $k>0$ iteration steps of GD have been carried out with a fixed stepsize $t$ such that $t\le \tfrac{1}{L\,(1+2u)^2}$ and $u$ satisfies $u\le \frac{a}{c\,+\,4a\,+\,4}$, where $c$ is the constant given in \cref{eq:boundsigma1}, for a certain $a$ with $0<a<1$. If the gradient satisfies 
\begin{align} \label{eq:bound_gradient_general}
\Vert \nabla f(\wh{\bx}^{(k-1)})\Vert\ge (1-a)^{-1} \, (2+4\,u+\sqrt{1-a}\,) \, \sqrt{n}\,c\,u,
\end{align}
then $f(\bz^{(k)}) \le f(\wh{\bx}^{(k-1)})$.
\end{Proposition}
\begin{proof}
Since $\nabla f$ is Lipschitz continuous with constant $L$, combining the updating rule with rounding errors \cref{eq:gd_floatp} and property \cref{eq:lc_ineq}, we have that
\begin{align*} 
f(\bz&^{(k+1)})- f(\wh{\bx}^{(k)})\nonumber\\
\le \,& -t\,\nabla f(\wh{\bx}^{(k)})^T(\nabla f(\wh{\bx}^{(k)})+\bsigma_1^{(k)}) \circ(\one+\bdelta_2^{(k)})\nonumber\\&\phantom{MM}+\tfrac{1}{2}\,L\,t^2\,\Vert (\nabla f(\wh{\bx}^{(k)})+\bsigma_1^{(k)}) \circ(\one+\bdelta_2^{(k)})\Vert^2\nonumber\\
=\, & \, -t\,\nabla f(\wh{\bx}^{(k)})^T(\nabla f(\wh{\bx}^{(k)}) \circ \bdelta_2^{(k)})-t\,\Vert \nabla f(\wh{\bx}^{(k)})\Vert^2-t\,\nabla f(\wh{\bx}^{(k)})^T\bsigma_1^{(k)}\nonumber\\
&\,\phantom{MM}  -t\,\nabla f(\wh{\bx}^{(k)})^T( \bsigma_1^{(k)} \circ \bdelta_2^{(k)})+\tfrac{1}{2}\,L\,t^2 \, \Vert (\nabla f(\wh{\bx}^{(k)})+\bsigma_1^{(k)}) \circ (\mathbf 1 + \bdelta_2^{(k)})\Vert^2.\nonumber
\end{align*}
Since $\vert \delta_{2,i}^{(k)}\vert\le 2u$ for $i = 1, \dots, n$, the upper bound for $f(\bz^{(k+1)})- f(\wh{\bx}^{(k)})$ becomes
\begin{align}
\label{eq:fzk_expansion}
 &t\,(2u-1)\,\Vert \nabla f(\wh{\bx}^{(k)})\Vert^2-t\,\nabla f(\wh{\bx}^{(k)})^T\bsigma_1^{(k)}+2u\,t\,\vert\nabla f(\wh{\bx}^{(k)})\vert^T\,\vert\boldsymbol{\sigma}_1^{(k)} \vert\nonumber\\
 &\phantom{MMM} +\tfrac{1}{2}\,L\,t^2\,(1+2u)^2\Vert \nabla f(\wh{\bx}^{(k)})+\bsigma_1^{(k)} \Vert^2\\
 & \le t\,(2u-1+\tfrac{1}{2}\,L\,t\,(1+2u)^2)\,\Vert \nabla f(\wh{\bx}^{(k)})\Vert^2+2u\,t\,\vert\nabla f(\wh{\bx}^{(k)})\vert^T\,\vert\boldsymbol{\sigma}_1^{(k)} \vert \nonumber\\
 &\phantom{MMM} +t\,(1-L\,t\,(1+2u)^2)\,\vert\nabla f(\wh{\bx}^{(k)})^T\bsigma_1^{(k)} \vert+\tfrac{1}{2}\,L\,t^2\,(1+2u)^2\,\Vert\bsigma_1^{(k)} \Vert^2.\nonumber
\end{align}
Based on \cref{eq:boundsigma1}, we have 
\begin{align} \label{eq:innerp_gradient_sigma}
 \vert\nabla f(\wh{\bx}^{(k)})^T\bsigma_1^{(k)} \vert \le \,&\vert\nabla f(\wh{\bx}^{(k)})\vert^T\,\vert\boldsymbol{\sigma}_1^{(k)} \vert\nonumber\\ \le \,&\sum_{i=1}^n \vert\nabla f(\wh{\bx}^{(k)})_i\vert\, (\vert\nabla f(\wh{\bx}^{(k)})_i\vert+1)\,c\,u\nonumber\\
 \le \,& c\,u\,(\Vert\nabla f(\wh{\bx}^{(k)})\Vert^2+ \sqrt{n}\,\Vert\nabla f(\wh{\bx}^{(k)})\Vert),
\end{align}
and 
\begin{align} \label{eq:norm_sigma1}
 \Vert \bsigma_1^{(k)} \Vert^2\le c^2\,u^2\,(\Vert\nabla f(\wh{\bx}^{(k)})\Vert^2+ 2\sqrt{n}\,\Vert\nabla f(\wh{\bx}^{(k)})\Vert+n).
\end{align}
Therefore, we obtain that $f(\bz^{(k+1)})-f(\wh{\bx}^{(k)})$ is bounded from above by
\begin{align}
& -t\,(\,(1-c\,u)-\tfrac{1}{2}\,L\,t\,(1+2u)^2(1-2\,c\,u+c^2\,u^2)-2u\,(1+c\,u)\,)\,\Vert \nabla f(\wh{\bx}^{(k)})\Vert^2\nonumber\\
& \phantom{MMM} +\sqrt{n}\,t\,c\,u\,(1+2u-L\,t\,(1+2u)^2(1-c\,u)\,)\,\Vert\nabla f(\wh{\bx}^{(k)})\Vert\nonumber\\
&\phantom{MMM}+\tfrac{1}{2}\,Lt^2(1+2u)^2\,n\,c^2u^2\nonumber\\
 &\le -t\,(\,(1-c\,u)(1-\tfrac{1}{2}\,L\,t\,(1+2u)^2(1-c\,u))-2u\,(1+c\,u)\,)\,\Vert \nabla f(\wh{\bx}^{(k)})\Vert^2\nonumber\\
 & \phantom{MMM} +\sqrt{n}\,t\,c\,u\,(1+2u-L\,t\,(1+2u)^2(1-c\,u)\,)\,\Vert\nabla f(\wh{\bx}^{(k)})\Vert+\tfrac{1}{2}\,t\,n\,c^2u^2.\nonumber
\end{align}
In the last inequality we have used that $Lt\,(1+2u)^2\le 1$. Note that $u\le \frac {a}{c\,+\,4a\,+\,4}$ and $a<1$ imply $0<c\,u<1$ which in turn gives $1-\tfrac{1}{2}\,L\,t\,(1+2u)^2(1-c\,u)\ge \tfrac{1}{2}$. Consequently, we achieve 
\begin{align} \label{eq:fz_general}
f(\bz^{(k+1)})-f(\wh{\bx}^{(k)}) \le \, 
&-t\,(\tfrac{1}{2}\,-\tfrac{1}{2}\,c\,u-2u-2\,c\,u^2\,)\,\Vert \nabla f(\wh{\bx}^{(k)})\Vert^2\\
 & \phantom{MMM}+\sqrt{n}\,t\,c\,u\,(1+2u)\,\Vert\nabla f(\wh{\bx}^{(k)})\Vert+\tfrac{1}{2}\,t\,n\,c^2u^2\nonumber.
\end{align}
To achieve $f(\bz^{(k+1)}) \le f(\wh{\bx}^{(k)})$, it is sufficient to have $\tfrac{1}{2}\,-\tfrac{1}{2}\,c\,u-2u-2\,c\,u^2>0$. Substituting $u\le \frac{a}{c\,+\,4a\,+\,4}$ into $\tfrac{1}{2}\,-\tfrac{1}{2}\,c\,u-2u-2\,c\,u^2$, based on the property $a>0$, one may check that $\tfrac{1}{2}\,-\tfrac{1}{2}\,c\,u-2u-2\,c\,u^2
\ge\tfrac{a}{2}\,\big(\tfrac{1}{a}-\tfrac{c^2+2\,(4a+4)\,c+4\,(4a+4)}{(c\,+\,4a\,+\,4)^2}\big)>\tfrac{1}{2}(1-a)$, 
which in turn yields an upper bound for $f(\bz^{(k+1)})-f(\wh{\bx}^{(k)})$:
\begin{align} \label{eq:converrate_u}
-\tfrac{1}{2}(1-a)\,t\,\Vert \nabla f(\wh{\bx}^{(k)})\Vert^2+\sqrt{n}\,t\,c\,u\,(1+2u)\,\Vert\nabla f(\wh{\bx}^{(k)})\Vert+\tfrac{1}{2}\,t\,n\,c^2u^2.
\end{align}
Property \cref{eq:bound_gradient_general} implies that $\big(\sqrt{1-a} \, \Vert \nabla f(\wh{\bx}^{(k-1)})\Vert-\frac{\sqrt{n}\,c\,u\,(1+2u)}{\sqrt{1-a}}\big)^2 \ge (\tfrac{1}{1-a}(1+2u)+1)^2\,n\,c^2\,u^2,$
which indicates that 
\begin{align}
(1-a)\,\Vert \nabla f(\wh{\bx}^{(k-1)})\Vert^2 &\ge 2\,n\,c\,u\,(1+2u)\,\Vert \nabla f(\wh{\bx}^{(k-1)})\Vert\nonumber\\
&\qquad\qquad+\tfrac{2}{1-a}(1+2u)\,n\,c^2\,u^2+n\,c^2\,u^2.\nonumber
\end{align}
Therefore, by applying \cref{eq:converrate_u} to the $k$th iteration step, we have that $f(\bz^{(k)}) \le f(\wh{\bx}^{(k-1)})$. 
\end{proof}
\cref{prop:boundonu} sheds light on the largest upper bound for the parameter $a$ that is required to guarantee monotonicity. Condition \cref{eq:bound_gradient_general} indicates that GD may only converge to a neighborhood of the optimal point due to the rounding errors. In particular, a smaller value of $a$ may allow GD to converge to a point that is closer to the optimal point.

Now let us not restrict ourselves to a specific rounding scheme and we look at the conditions that guarantee the monotonicity of \cref{eq:gd_floatp}. Based on \cref{eq:gd_floatp2}, we denote by
\begin{equation} \label{eq:theta_f}
\theta^{(k)}:=f(\wh{\bx}^{(k+1)})-f(\bz^{(k+1)}),
\end{equation}
the effect of the third roundoff error $\bdelta_3^{(k)}$ on the objective function value. In \cref{lem:boundforfloating_point}, we will show that a smaller bound on $u$ leads to a tighter bound on the convergence rate of GD. By setting $0<a<\frac{1}{2}$ instead of $0<a<1$ in \cref{prop:boundonu}, we can prove the following result that links the monotonicity of GD to the values of $u$ and $t$. 
\begin{Lemma} \label{thm:monotone}
Let the objective function f satisfy \cref{model:convexfun} and $u\le \frac{a}{c\,+\,4a\,+\,4}$ with $0<a<\tfrac{1}{2}$. Assume that $k>0$ iteration steps of GD have been carried out with a fixed stepsize $t$ such that $t\le \tfrac{1}{L\,(1+2u)^2}$. If the gradient satisfies 
\begin{align} \label{eq:bound_gradient}
\Vert \nabla f(\wh{\bx}^{(k-1)})\Vert\ge a^{-1}\,(2+4\,u+\sqrt{a}\,)\,\sqrt{n}\,c\,u
\end{align} and $u$ satisfies
\begin{align}
u\le \tfrac{1}{4} (1-2a)\,t \ \frac{\Vert \nabla f(\wh{\bx}^{(k-1)})\Vert^2}{\Vert \nabla f(\wh{\bx}^{(k)})\Vert \, \Vert \bz^{(k)} \Vert},\label{eq:condition2}
\end{align} where $\bz^{(k)}$ is as in \cref{eq:gd_floatp1}, then $f(\wh{\bx}^{(k)}) \le f(\wh{\bx}^{(k-1)})$. 
\end{Lemma}
\begin{proof}
We proceed similarly as for \cref{prop:boundonu}. Combining the property $\tfrac{1}{2}\,-\tfrac{1}{2}\,c\,u-2u+2\,c\,u^2 \ge \tfrac{1}{2}(1-a)$ with \cref{eq:fz_general}, we have that $f(\bz^{(k)})-f(\wh{\bx}^{(k-1)})$ is bounded from above by
\begin{align*}
 -\tfrac{1}{2}(1-a)\,t\,\Vert \nabla f(\wh{\bx}^{(k-1)})\Vert^2+\sqrt{n}\,t\,c\,u\,(1+2u)\,\Vert\nabla f(\wh{\bx}^{(k-1)})\Vert+\tfrac{1}{2}\,t\,n\,c^2u^2.
\end{align*}
One may check that property \cref{eq:bound_gradient} indicates that 
\begin{align*}
\sqrt{n}\,t\,c\,u\,(1+&2u)\,\Vert\nabla f(\wh{\bx}^{(k-1)})\Vert\!+\tfrac{1}{2}\,t\,n\,c^2u^2\nonumber\\
&\le \tfrac{2\,(1\!+\!2u)\,(2\,(1+2u)+\sqrt{a}\,)\,a+a^2}{2\,(2\,(1+2u)+\sqrt{a}\,)^2} \, \Vert\nabla f(\wh{\bx}^{(k-1)})\Vert^2\\
&=\tfrac{a}{2} \, \big(\tfrac{4\,(1+2u)^2+2\sqrt{a}\,(1+2u)+a}{4\,(1+2u)^2+4\,\sqrt{a}\,(1+2u)+a}\big)\,\Vert\nabla f(\wh{\bx}^{(k)})\Vert^2<\tfrac{a}{2} \, \Vert\nabla f(\wh{\bx}^{(k-1)})\Vert^2,
\end{align*}
which implies that 
\begin{align} \label{eq:ineq_rounded}
f(\wh{\bx}^{(k}) \le \,&f(\wh{\bx}^{(k-1)})-\tfrac{1}{2}(1-2a)\,t\,\Vert \nabla f(\wh{\bx}^{(k-1)})\Vert^2+\theta^{(k-1)}.
\end{align}
Since $f$ is convex, based on \cref{eq:lc_ineq} and \cref{eq:theta_f}, we obtain the following inequality
\begin{align*}
f(\wh{\bx}^{(k-1)})& -f(\wh{\bx}^{(k)})
\ge \,
\nabla f(\wh{\bx}^{(k)})(\bz^{(k)}-\wh{\bx}^{(k)})+\tfrac{1}{2}(1-2a)\,t\, \Vert \nabla f(\wh{\bx}^{(k-1)})\Vert^2\nonumber\\
\ge \, & -2u\, \Vert\bz^{(k)} \Vert \Vert\nabla f(\wh{\bx}^{(k)}) \Vert+\tfrac{1}{2}(1-2a)\,t\,\Vert \nabla f(\wh{\bx}^{(k-1)})\Vert^2\geq 0,
\end{align*}
where the last inequality is obtained in view of $a<\frac{1}{2}$ and \cref{eq:condition2}.
\end{proof}

Comparing \cref{theorem:convergencerate} and \cref{thm:monotone}, a slightly smaller $t$ is chosen in \cref{thm:monotone} to compensate the harmful effect of rounding errors, i.e., $t\le \tfrac{1}{L\,(1+2u)^2}$ in \cref{thm:monotone} instead of $t\le \frac{1}{L}$ in \cref{theorem:convergencerate}. Condition \cref{eq:condition2} may be viewed as either an upper bound on $u$ or a lower bound on $t$,
depending on the setting. When implementing GD with a fixed $t$, a smaller $a$ increases the upper bound in \cref{eq:condition2} and implies a smaller $u$, which in turn makes condition \cref{eq:condition2} easier to satisfy. When $\theta^{(k-1)}\le0$ it is sufficient to have \cref{eq:bound_gradient_general} to get $f(\wh{\bx}^{(k)}) \le f(\wh{\bx}^{(k-1)})$. 

We now address the generalization of \cref{theorem:convergencerate} for the updating rule with rounding errors \cref{eq:gd_floatp}. The core idea is to adjust the strategy used in the proof of \cite[Thm.~2.1.14, Cor.~2.1.2]{nesterov2003introductory} to our setting. If $u$ is small enough, we ensure a similar convergence rate $\mathcal{O}(1/k)$ and we show that a better multiplicative constant may be obtained when the accumulated absolute rounding errors are in a descent direction. 
 
Before stating the result we introduce the upper bound for the distance between the iterates of GD and the minimizer $\bx^*$, and the best approximation to the optimal value: 
\begin{align*}
 \chi:= \max_{j=0,\dots,k} \Vert \wh{\bx}^{(j)}-\bx^{*} \Vert,\qquad \zeta:= \min_{j=0,\dots,k} f( \wh{\bx}^{(j)})-f(\bx^{*}).
 \end{align*}
In the following theorem, we require $u$ to satisfy a bound slightly stricter than \cref{eq:condition2}; in particular this guarantees the monotonicity of the GD iterations. Moreover, we will introduce a quantity $\alpha_j$ that interprets the relation between $\theta^{(j)}$, $\chi$, and $\zeta$. Furthermore, $\alpha_j$ has the same sign as and is proportional to $\theta^{(j)}$. This indicates that $\alpha_j$ has the same effect as $\theta^{(j)}$ on the convergence of GD. We will discuss more details after \cref{lem:boundforfloating_point}.
\begin{theorem} \label{lem:boundforfloating_point}
Let the objective function f satisfy \cref{model:convexfun} and $u\le \frac{a}{c\,+\,4a\,+\,4}$ with $0<a<\tfrac{1}{2}$. Assume that $k>0$ iteration steps of GD have been carried out with a fixed stepsize $t$ such that $t\le \tfrac{1}{L\,(1+2u)^2}$. If the gradient satisfies \cref{eq:bound_gradient} and $u$ satisfies
\begin{align} \label{eq:condition3}
u\le \tfrac{1}{4} \, (1-2a)\,t \ \frac{\zeta^2}{\chi^2\,\Vert \nabla f(\wh{\bx}^{(j)})\Vert \, \Vert \bz^{(j)} \Vert}, \quad j=0,\dots,k-1,
\end{align}
then 
\begin{align}
&f(\wh{\bx}^{(k)})-f(\bx^*) \le \frac{2L\chi^2}{4+L\,t\,(1-2a)\sum_{j=0}^{k-1} (1-\alpha_j)},\qquad \alpha_j:=\frac{2\,\chi^2\,\theta^{(j)}}{t\,(1-2a)\zeta^2}.
\label{eq:fx_ineq}
\end{align} 
\end{theorem}
\begin{proof}
We follow the main line of the proof in \cite[Thm.~2.1.14, Cor.~2.1.2]{nesterov2003introductory}. In view of \cref{eq:lc_ineq}, we have 
\begin{align}
\zeta\le f(\wh{\bx}^{(k)})-f(\bx^*) \le \,&\nabla f(\wh{\bx}^{(k)})^T(\wh{\bx}^{(k)}-\bx^*) \le \,\Vert\nabla f(\wh{\bx}^{(k)}) \Vert \, \chi.
\label{eq:convexity}
\end{align}
Based on \cref{eq:convexity}, it is easy to check that \cref{eq:condition3} implies \cref{eq:condition2}. Together with \cref{eq:ineq_rounded}, we have
\begin{align*}
f(\wh{\bx}^{(k+1)})-f(\bx^*)
 & \le \, f(\wh{\bx}^{(k)})-f(\bx^*)-\tfrac{1}{2}(1-2a)\, t\,\Vert \nabla f(\wh{\bx}^{(k)}) \Vert^2+\theta^{(k)}\nonumber\\
 &\le \, f(\wh{\bx}^{(k)})-f(\bx^*)-\tfrac{(1-2a)\,t}{2\,\chi^2}(f(\wh{\bx}^{(k)})-f(\bx^*))^2+\theta^{(k)}.
\end{align*}
Dividing both sides by $(f(\wh{\bx}^{(k+1)})-f(\bx^*))\,(f(\wh{\bx}^{(k)})-f(\bx^*))$, we obtain
\begin{align}
\frac{1}{f(\wh{\bx}^{(k)})-f(\bx^*)}\le \,& \frac{1}{f(\wh{\bx}^{(k+1)})-f(\bx^*)}-\frac{(1-2a)\,t}{2\,\chi^2}\frac{f(\wh{\bx}^{(k)})-f(\bx^*)}{f(\wh{\bx}^{(k+1)})-f(\bx^*)}\nonumber\\&\phantom{MMM} +\frac{\theta^{(k)}}{(f(\wh{\bx}^{(k+1)})-f(\bx^*)) \, (f(\wh{\bx}^{(k)})-f(\bx^*))}\nonumber\\
\underset{\textrm{\cref{thm:monotone}}}{\qquad\quad \le} &\,\frac{1}{f(\wh{\bx}^{(k+1)})-f(\bx^*)}-\frac{(1-2a)\,t}{2\,\chi^2}\nonumber\\&\phantom{MMM}
+\frac{\theta^{(k)}}{(f(\wh{\bx}^{(k+1)})-f(\bx^*)) \, (f(\wh{\bx}^{(k)})-f(\bx^*))}\nonumber\\
\le \,&\frac{1}{f(\wh{\bx}^{(k+1)})-f(\bx^*)}-\frac{(1-2a)\,t}{2\,\chi^2}+\frac{\theta^{(k)}}{\zeta^2}.\nonumber
\end{align}
Note that the convexity of $f$ and \cref{eq:condition3} yield $\frac{(1-2a)\,t}{2\,\chi^2}-\frac{\theta^{(j)}}{\zeta^2}\ge 0$, for $j=0,\dots, k-1$. Expanding the recursion $k$ times, we obtain
\begin{align}
\frac{1}{f(\wh{\bx}^{(k)})-f(\bx^*)} & \ge \frac{1}{f(\bx^{(0)})-f(\bx^*)}+\frac{t\,(1-2a)}{2\,\chi^2}\sum_{j=0}^{k-1} \Big(1- \frac{2\,\chi^2\,\theta^{(j)}}{t\,(1-2a)\zeta^2}\Big)\nonumber\\&
=\frac{1}{f(\bx^{(0)})-f(\bx^*)}+\frac{t\,(1-2a)}{2\,\chi^2}\sum_{j=0}^{k-1} (1- \alpha_j).
\label{eq:fracversionfxk}
\end{align}
Property \eqref{eq:lc_ineq} and $\nabla f(\bx^*)=0$ imply
\begin{align}
\frac{1}{f(\bx^{(0)})-f(\bx^*)}\ge \frac{2}{L\,\Vert \bx^{(0)}-\bx^*\Vert^2}.\label{eq:f0}
\end{align}
Applying \eqref{eq:f0} to \eqref{eq:fracversionfxk}, we obtain $\frac{1}{f(\wh{\bx}^{(k)})-f(\bx^*)}\ge \,\frac{2}{L\,\Vert \bx^{(0)}-\bx^*\Vert^2}+\frac{t\,(1-2a)}{2\,\chi^2}\sum_{j=0}^{k-1} (1-\alpha_j).$
Therefore, we have \cref{eq:fx_ineq}.
\end{proof}
\cref{lem:boundforfloating_point} holds for both deterministic and stochastic rounding methods. We remark that, after accounting for the rounding errors in evaluating \cref{eq:gd_floatp0} and \cref{eq:gd_floatp1}, we cannot obtain a bound in \cref{lem:boundforfloating_point} that is identical to \cref{theorem:convergencerate} even with $\alpha_j =0$ for all $ j=0,\dots, k-1$ in \cref{eq:fx_ineq}.
Both \cref{thm:monotone} and \cref{lem:boundforfloating_point} study the worst-case scenario with respect to $\bsigma_1$ and $\bsigma_2$, where all the rounding errors in \cref{eq:gd_floatp0,eq:gd_floatp1} are in an ascent direction. The condition $\theta^{(j)}<0$ indicates that the accumulated absolute rounding errors in \cref{eq:gd_floatp2} are oriented towards a descent direction. 

The control of the quantities $\alpha_j$ may be realized by using different deterministic strategies. For instance, to ensure $\alpha_j<0$ (i.e., $\theta^{(j)}<0$) one would need to switch the rounding scheme between floor and ceiling to match the condition:
$\sign(\nabla f(\wh{\bx}^{(j)}))=-\sign( \bdelta_3^{(j-1)} \circ \bx^{(j)})$, where the $\sign$ function is applied component-wise. This is equivalent to setting $\varepsilon=1$ in $\srh$. Therefore, switching between floor and ceiling will lead to a large rounding bias, which may cause oscillations as choosing a large $t$ in exact arithmetic. We will demonstrate this experimentally in \cref{sec:simulation} by using a large value of $\varepsilon$ in $\ssrh$ (cf.~\cref{fig:f1b,fig:f2b}). Additionally, in \cref{sec:s3}, we propose a bound on $\varepsilon$ ($\le0.5$), which may lead to a faster convergence of GD than $\csr$ (cf.~\cref{lem:expecd_srhsigned} and \cref{thm:modifiedsrh}). Consequently, a better way to analyze and possibly control the quantities $\alpha_j$ is to rely on stochastic rounding methods. In the next subsection we show that we can eliminate the harmful effect of $\alpha_j$ in the expectation, by using $\csr$. 
\subsection{On stochastic rounding: Scenario 1 (no stagnation)}
We will investigate the impact of stochastic rounding errors on the convergence of GD for general optimization problems fulfilling \cref{eq:sigma_1_c2}, then we extend it to optimization problems satisfying \cref{eq:sigma_1_c1} with $\csr$. 
Before we start our analysis, we recall a basic property of conditional expectation. 
For random variables $X$, $Y$, and $Z$, we have \cite[(10.40)]{steyer2017probability}
\begin{align} \label{eq:towerprop}
\expt\,[\,\expt\,[\,X\ \big\vert \ Y,\, Z\,]\ \big\vert \ Y\,]=\expt\,[\,X\ \big\vert \ Y\,].
\end{align}
Based on this property, we show that using $\csr$, the monotonicity of GD is guaranteed for condition \cref{eq:sigma_1_c2}.
\begin{theorem} \label{thm:conv-exp_csr}
Let the objective function f satisfy \cref{model:convexfun} and $u\le \frac{a}{c\,+\,4a\,+\,4}$ for an $a>0$. Assume that $k>0$ iteration steps of GD have been carried out with a fixed stepsize $t$ such that $t\le \tfrac{1}{L\,(1+2u)^2}$ satisfying condition \cref{eq:conditon_1} and both \cref{eq:gd_floatp1} and \cref{eq:gd_floatp2} are evaluated using $\csr$. 

(i) Under condition \cref{eq:sigma_1_c2}, if it holds that 
\begin{align} \label{eq:normf_csr1}
 \expt\,[\, \Vert \nabla f(\wh{\bx}^{(j)})\Vert\,]\ge a^{-1} \, (2+\sqrt{a}) \, \sqrt{n}\,c\,u
 \end{align} for all $j=0,\dots, k-1$ and $a< \frac{1}{2}$, then 
\begin{equation}
\expt\,[\,f(\wh{\bx}^{(k)})-f(x^*)\,]
\le \frac{2L\chi^2}{4+L\,t\,k\,(1-2a)}.
\label{eq:expectedbound_csr1}
\end{equation}

(ii) Under condition \cref{eq:sigma_1_c1}, if it holds that\begin{align} \label{eq:normf_csr2}
\expt\,[\, \Vert \nabla f(\wh{\bx}^{(j)})\Vert\,]\ge 3 \, a^{-1} \, \sqrt{n}\,c\,u
\end{align}
for all $j=0,\dots, k-1$ and $a< \frac12 \sqrt{2}$, then
\begin{equation} \label{eq:expectedbound_csr2}
\expt\,[\,f(\wh{\bx}^{(k)})-f(x^*)\,]
\le \frac{2L\chi^2}{4+L\,t\,k\,(1-2a^2)}.
\end{equation}
\end{theorem} 
\begin{proof}
Updating rule \cref{eq:gd_floatp1} can be represented by using the absolute rounding errors, 
\begin{align} \label{eq:gd_abse}
\bz^{(k+1)} &=\wh{\bx}^{(k)}-t\,(\nabla f(\wh{\bx}^{(k)})+\bsigma_1^{(k)})-\bsigma_2^{(k)}.\end{align}
Based on the fact that $t\,(\nabla f(\wh{\bx}^{(k)})+\bsigma_1^{(k)})+\bsigma_2^{(k)}= t\,(\nabla f(\wh{\bx}^{(k)})+\bsigma_1^{(k)}) \circ (\mathbf 1 + \bdelta_2^{(k)})$, combining \cref{eq:gd_abse} with property \cref{eq:lc_ineq} and proceeding similarly to \cref{eq:fzk_expansion}, we obtain an upper bound for $f(\bz^{(k+1)})-f(\wh{\bx}^{(k)}):$
\begin{align}
&-t\,(1-\tfrac{1}{2}\,L\,t\,(1+2u)^2)\,\Vert\nabla f(\wh{\bx}^{(k)})\Vert^2-\nabla f(\wh{\bx}^{(k)})^T\bsigma_2^{(k)}\nonumber\\&\phantom{MMM}+t\,(1-L\,t\,(1+2u)^2)\,\vert \nabla f(\wh{\bx}^{(k)})^T\bsigma_1^{(k)} \vert+\tfrac{1}{2}\,L\,t^2\,(1+2u)^2\Vert\bsigma_1^{(k)} \Vert^2.\nonumber
\end{align}

{\bf Part (i)}: Taking the expectation and in view of \cref{eq:sigma_1_c2} and \cref{eq:norm_sigma1}, we have that 
\begin{align} \label{eq:zk_csr}
&\expt\,[\,f(\bz^{(k+1)})\,]-\expt\,[\,f(\wh{\bx}^{(k)})\,]\\
&\le \sqrt{n}\,t\,c\,u\,(1-L\,t\,(1\!+\!2u)^2+L\,t\,(1\!+\!2u)^2c\,u)\,\expt\,[\,\Vert\nabla f(\wh{\bx}^{(k)})\Vert\,]+\tfrac{1}{2}\,t\,n\,c^2u^2\nonumber\\&-t\,(1\!-c\,u-\tfrac{1}{2}\,L\,t\,(1\!+\!2u)^2\,(1-c\,u)^2\,)\,\expt\,[\,\Vert\nabla f(\wh{\bx}^{(k)})\Vert^2\,]-\expt\,[\,\nabla f(\wh{\bx}^{(k)})^T\bsigma_2^{(k)}\,].\nonumber
\end{align}
Based on \cref{eq:towerprop}, when $\csr$ is applied for evaluating $\bsigma_2^{(k)}$, given $w_1$ in the finite set $\mathcal{S}_1$ of possible values of $\nabla f(\wh{\bx}^{(k)})_i$, we have 
\begin{align}
 \expt\,[\,\sigma_{2,i}^{(k)}\ \big\vert \ \nabla f(\wh{\bx}^{(k)})_i=w_1\,]= \expt\,[\,\expt\,[\,\sigma_{2,i}^{(k)}\ \big\vert \ \nabla f(\wh{\bx}^{(k)})_i, \, \sigma_{1,i}^{(k)}\,]\ \big\vert \ \nabla f(\wh{\bx}^{(k)})_i=w_1\,]=0.\nonumber
\end{align}
The law of total expectation gives 
\begin{align} \label{eq:sigma2_csr}
 \expt\,[\,\sigma_{2,i}^{(k)} \, \nabla f(\wh{\bx}^{(k)})_i\,] \!=\!\!\!\!\sum_{w_1\in \mathcal{S}_1} \expt\,[\,\sigma_{2,i}^{(k)}\,w_1\ \big\vert \ \nabla f(\wh{\bx}^{(k)})_i=w_1\,]\,P(\nabla f(\wh{\bx}^{(k)})_i=w_1)=0.
\end{align}
The property $c\,u\le a<1$ implies $1-L\,t\,(1+2u)^2+L\,t\,(1+2u)^2\,c\,u\le 1-(1-a)L\,t\,(1+2u)^2\le 1$.
Therefore, substituting $t\le \tfrac{1}{L\,(1+2u)^2}$ into \cref{eq:zk_csr}, we achieve the following upper bound for $\expt\,[\, f(\bz^{(k+1)})\,]-\expt\,[\,f(\wh{\bx}^{(k)})\,]$:
\begin{align}
-\tfrac{1}{2}\,t\,\expt\,[\,\Vert\nabla f(\wh{\bx}^{(k)})\Vert^2\,]+&\tfrac{1}{2}\,t\,a\,\expt\,[\,\Vert\nabla f(\wh{\bx}^{(k)})\Vert^2\,]+\sqrt{n}\,t\,c\,u\,\expt\,[\,\Vert\nabla f(\wh{\bx}^{(k)})\Vert\,]+\tfrac{1}{2}\,t\,n\,c^2u^2.\nonumber
\end{align}
Property \cref{eq:normf_csr1} and Jensen's inequality~\cite[Lemma 5.3.1]{kuczma2009introduction} indicate that
\[
\tfrac{1}{2}\,a\,\expt\,[\,\Vert\nabla f(\wh{\bx}^{(k)})\Vert^2\,]\ge\sqrt{n}\,c\,u\,\expt\,[\,\Vert\nabla f(\wh{\bx}^{(k)})\Vert\,]+\tfrac{1}{2}\,n\,c^2u^2,
\]
so that
\begin{align} \label{eq:expectedfxk_csr}
\expt\,[\, f(\bz^{(k+1)})\,]
\le \,\expt\,[\,f(\wh{\bx}^{(k)})\,]-\tfrac{1}{2}\,t\,(1-2a)\,\expt\,[\,\Vert\nabla f(\wh{\bx}^{(k)})\Vert^2\,].
\end{align}
The property $a< \tfrac{1}{2}$ implies $1-2a>0$. 
When $\bsigma_3^{(k)}$ is generated by $\csr$, we get zero mean independent errors \cite[Lemma 5.2]{connolly2021stochastic}, which implies $\expt\,[\,f(\wh{\bx}^{(k+1)})\,]-\expt\,[\,f(\bz^{(k+1)})\,]=0$. In view of \cref{eq:convexity}, we have $\expt\,[\,\Vert\nabla f(\wh{\bx}^{(k)}) \Vert^2\,]\ge \chi^{-2}\,\expt\,[\,(f(\wh{\bx}^{(k)})-f(\bx^*))^2\,].$ Therefore, substituting this into \cref{eq:expectedfxk_csr} and on the basis of Jensen's inequality, we obtain
\begin{align} \label{eq:expectedineq}
\expt\,[\,f(\wh{\bx}^{(k+1)})\,]
\le \,&\expt\,[\,f(\wh{\bx}^{(k)})\,]-t\,(1-2a)\,\frac{\expt\,[\,f(\wh{\bx}^{(k)})-f(\bx^*)\,]^2}{2\,\chi^2}.
\end{align}
Dividing both sides of \cref{eq:expectedineq} by $\expt\,[\,f(\wh{\bx}^{(k+1)})-f(\bx^*)\,]\,\expt\,[\,f(\wh{\bx}^{(k)})-f(\bx^*)\,]$, we obtain
\begin{align} \label{eq:expected_csr}
\frac{1}{\expt\,[\,f(\wh{\bx}^{(k)})-f(\bx^*)\,]}
\le \,&\frac{1}{\expt\,[\,f(\wh{\bx}^{(k+1)})\!-\!f(\bx^*)\,]}-\frac{t\,(1\!-\!2a)}{2\chi^2}\frac{\expt\,[\,f(\wh{\bx}^{(k)})-f(\bx^*)\,]}{\expt\,[\,f(\wh{\bx}^{(k+1)})\!-\!f(\bx^*)\,]}\nonumber\\
\underset{\text{\cref{eq:expectedfxk_csr}}}{\quad\le}\,&\frac{1}{\expt\,[\,f(\wh{\bx}^{(k+1)})\!-\!f(\bx^*)\,]}-\frac{t\,(1-2a)}{2\chi^2}.
\end{align}
Expanding the recursion of \cref{eq:expected_csr} until the $k$th iteration step and based on \cref{eq:f0}, we obtain \[\frac{1}{\expt\,[\,f(\wh{\bx}^{(k)})-f(\bx^*)\,]}\ge \frac{2}{L\,\Vert \bx^{(0)}-\bx^*\Vert^2}+\frac{t}{2\chi^2}k(1-2a),\] concluding the claim.

\textbf{Part (ii)}: 
On the basis of \cref{eq:sigma_1_c1}, \cref{eq:sigma2_csr}, and the property $L\,t\,(1+2u)^2\le 1$, we get 
\begin{align} 
\expt\,[\,f(\bz^{(k+1)})\,]\le \,&\expt\,[\,f(\wh{\bx}^{(k)})\,]-\tfrac{1}{2}\,t\,\expt\,[\,\Vert\nabla f(\wh{\bx}^{(k)})\Vert^2\,]+\tfrac{1}{2}\,t\,\expt\,[\,\Vert\bsigma_1^{(k)} \Vert^2\,].\nonumber
\end{align}
Taking the expectation of \cref{eq:norm_sigma1} and substituting it into the above expression, on the basis of $c\,u<a$, we achieve that $\expt\,[\,f(\bz^{(k+1)})\,]-\expt\,[\,f(\wh{\bx}^{(k)})\,]$ is bounded from above by 
\begin{align} \label{eq:omegak}
& -\tfrac{1}{2}\,t\,\expt\,[\,\Vert\nabla f(\wh{\bx}^{(k)})\Vert^2\,]+\tfrac{1}{2}\,t\,c^2\,u^2\,(\expt\,[\,\Vert\nabla f(\wh{\bx}^{(k)})\Vert^2\,]+ 2\sqrt{n}\,\expt\,[\,\Vert\nabla f(\wh{\bx}^{(k)})\Vert\,]+n)\nonumber\\
&\le-\tfrac{1}{2}\,t\,\expt\,[\,\Vert\nabla f(\wh{\bx}^{(k)})\Vert^2\,]+\tfrac{1}{2}\,t\,a^2\,\expt\,[\,\Vert\nabla f(\wh{\bx}^{(k)})\Vert^2\,]+t\,a\,\sqrt{n}\,c\,u\,\expt\,[\,\Vert\nabla f(\wh{\bx}^{(k)})\Vert\,]+\tfrac{1}{2}\,t\,n\,c^2\,u^2.
\end{align}
On the basis of property \cref{eq:normf_csr2} and Jensen's inequality, we have that $a\,\sqrt{n}\,c\,u\,\expt\,[\,\Vert\nabla f(\wh{\bx}^{(k)})\Vert\,]+\tfrac{1}{2}\,n\,c^2\,u^2\le \tfrac{1}{2}\,a^2\,\expt\,[\,\Vert\nabla f(\wh{\bx}^{(k)})\Vert^2\,].$ This results in \[\expt\,[\,f(\bz^{(k+1)})\,]-\expt\,[\,f(\wh{\bx}^{(k)})\,]\le-\tfrac{1}{2}\,t\,(1-2a^2)\,\expt\,[\,\Vert\nabla f(\wh{\bx}^{(k)})\Vert^2\,].\]
Finally, following the same argument as that used to obtain \cref{eq:expectedbound_csr1}, we are able to obtain \cref{eq:expectedbound_csr2}.
\end{proof}
When $c \ne 0$, as shown by both \cref{eq:normf_csr1} and \cref{eq:normf_csr2}, GD may only converge sublinearly to a neighborhood of the optimal point, while it may converge to the optimal point when $c=0$ (gradients are evaluated without errors). In particular, among the three rounding errors in \cref{eq:gd_floatp}, $\bsigma_1$ determines the achievable accuracy of the returned approximate solution. When using an objective function that satisfies \cref{eq:sigma_1_c1}, the bound for \cref{eq:normf_csr2} is stricter than \cref{eq:normf_csr1}, indicating that GD may converge to a point that is closer to the optimal point under condition \cref{eq:sigma_1_c1} than \cref{eq:sigma_1_c2}. A comparison of \cref{eq:expectedbound_csr1,eq:expectedbound_csr2,eq:convergencerate_general} demonstrates that the convergence bound for GD using $\csr$ is less sharp than the one obtained by the exact arithmetic, but it may be sharper than the one obtained by deterministic rounding methods \cref{eq:fx_ineq}. Furthermore, we show that we may achieve a stricter bound of convergence using $\srh$.

\begin{Corollary} \label{coro:conv-exp_srh}
Under the same assumptions as in \cref{thm:conv-exp_csr}, let \cref{eq:gd_floatp1} and \cref{eq:gd_floatp2} be evaluated using $\srh$ and $\csr$, respectively. 

(i) Under condition \cref{eq:sigma_1_c2}, if it holds that  
 \begin{align} \label{eq:normf_srh1}
 \expt\,[\, \Vert \nabla f(\wh{\bx}^{(j)})\Vert\,]\ge a^{-1}\,(2+\sqrt{a}+4\,\varepsilon\,u)\,\sqrt{n}\,c\,u
 \end{align} for all $j=0,\dots, k-1$, then there exists a $0<b\le 2\,\varepsilon\,u$ such that
\begin{equation}
\expt\,[\,f(\wh{\bx}^{(k)})-f(x^*)\,]
\le \frac{2L\chi^2}{4+L\,t\,k\,(1+2 b-2a)}.
\label{eq:expectedbound_srh}
\end{equation}

(ii) Under condition \cref{eq:sigma_1_c1}, if it holds that\begin{align} \label{eq:normf_srh2}
\expt\,[\, \Vert \nabla f(\wh{\bx}^{(j)})\Vert\,]\ge a^{-1}\,3\,\sqrt{n}\,c\,u
\end{align}
for all $j=0,\dots, k-1$ and $a< \frac12 \sqrt{2}$, then there exists a $0<b\le 2\,\varepsilon\,u$ such that 
\begin{equation}\label{eq:srh_result2}
\expt\,[\,f(\wh{\bx}^{(k)})-f(x^*)\,]
\le \frac{2L\chi^2}{4+L\,t\,k\,(1+2 b-2a^2)}.
\end{equation}
\end{Corollary}
\begin{proof}
\textbf{Part (i)}: According to \cref{lem:exp_deltasrh}, when $\bdelta_2^{(k)}$ is evaluated by $\srh$, we have 
$0\le \expt\,[\,\delta_{2,i}^{(k)}\ \big\vert \ \nabla f(\wh{\bx}^{(k)})_i, \sigma_{1,i}^{(k)}\,]\le 2\,\varepsilon\,u$. Following the similar argument as for \cref{eq:sigma2_csr}, we obtain
\begin{align*}
 \expt\,[\,\delta_{2,i}^{(k)}\,\sigma_{1,i}^{(k)} \, \nabla f(\wh{\bx}^{(k)})_i\,]
 &\le 2\,\varepsilon\,u\,\expt\,[\,\vert\,\sigma_{1,i}^{(k)} \, \nabla f(\wh{\bx}^{(k)})_i\,\vert\,]. 
\end{align*}
Analogously, we have 
\begin{align*}
 \expt\,[\,\delta_{2,i}^{(k)}\,\nabla f(\wh{\bx}^{(k)})_i^2\,]= \expt\,[\,\expt\,[\,\delta_{2,i}^{(k)}\,\nabla f(\wh{\bx}^{(k)})_i^2\ \big\vert \ \nabla f(\wh{\bx}^{(k)})_i\,]\,]\le 2\,\varepsilon\,u\,\expt\,[\,\nabla f(\wh{\bx}^{(k)})_i^2\,].
\end{align*}
Define $b:= \ds \min_{j=0,\dots,k} \tfrac{\sum_{i=1}^{n}\expt\,[\,\delta_{2,i}^{(j)}\,\nabla f(\wh{\bx}^{(j)})_i^2\,]}{\expt\,[\,\Vert\nabla f(\wh{\bx}^{(j)})\Vert^2\,]}$, then we have $0<b\le 2\,\varepsilon\,u$.
In view of \cref{eq:gd_abse,eq:lc_ineq}, we obtain an upper bound for $\expt\,[\,f(\bz^{(k+1)})\,]-\expt\,[\,f(\wh{\bx}^{(k)})\,]$:
\begin{align*} 
&-t\,(1-\tfrac{1}{2}\,L\,t\,(1+2u)^2)\,\expt\,[\,\Vert\nabla f(\wh{\bx}^{(k)})\Vert^2\,]-\expt\,[\,\nabla f(\wh{\bx}^{(k)})^T\bsigma_2^{(k)}\,]\nonumber\\&\phantom{MM} -t\,(1-L\,t\,(1+2u)^2)\,\expt\,[\,\nabla f(\wh{\bx}^{(k)})^T\bsigma_1^{(k)}\,]+\tfrac{1}{2}\,L\,t^2\,(1+2u)^2\,\expt\,[\,\Vert\bsigma_1^{(k)} \Vert^2\,].
\end{align*}
Therefore, by replacing the absolute error $\bsigma_2^{(k)}$ by the corresponding relative error expression in the above function, we have that $\expt\,[\,f(\bz^{(k+1)})\,]-\expt\,[\,f(\wh{\bx}^{(k)})\,]$ is bounded from above by
\begin{align*} 
&\phantom{M}-t\sum_{i=1}^{n}\expt\,[\,\nabla f(\wh{\bx}^{(k)})_i\sigma_{1,i}^{(k)}\,\delta_{2,i}^{(k)}\,]-t\sum_{i=1}^{n}\expt\,[\,\nabla f(\wh{\bx}^{(k)})_i^2\,\delta_{2,i}^{(k)}]\\
&\phantom{M}-t\,(1-\tfrac{1}{2}\,L\,t\,(1+2u)^2)\,\expt\,[\,\Vert\nabla f(\wh{\bx}^{(k)})\Vert^2\,]+\tfrac{1}{2}\,L\,t^2\,(1+2u)^2\,\expt\,[\,\Vert\bsigma_1^{(k)} \Vert^2\,]\\
&\phantom{M} -t\,(1-L\,t\,(1+2u)^2)\,\expt\,[\,\nabla f(\wh{\bx}^{(k)})^T\bsigma_1^{(k)}\,]\\
&\le -t\,(1-\tfrac{1}{2}\,L\,t\,(1+2u)^2+b)\,\expt\,[\,\Vert\nabla f(\wh{\bx}^{(k)})\Vert^2\,]\\
& \phantom{M} +\tfrac{1}{2}\,L\,t^2\,(1+2u)^2\,\expt\,[\,c^2\,u^2\,\Vert\nabla f(\wh{\bx}^{(k)})\Vert^2+ c^2\,u^2\,2\sqrt{n}\,\Vert\nabla f(\wh{\bx}^{(k)})\Vert+n\,c^2u^2\,]\\
& \phantom{M} +t\,(1\!-\!L\,t\,(1+2u)^2+2\,\varepsilon\,u)(c\,u\,\expt\,[\,\Vert\nabla f(\wh{\bx}^{(k)})\Vert^2\,]+\! c\,u\,\sqrt{n}\,\expt\,[\,\Vert\nabla f(\wh{\bx}^{(k)})\Vert\,])
\\
&\le-t\,(1\!-\!c\,u\!-\!\tfrac{1}{2}L\,t\,(1\!+\!2u)^2(1\!-\!c\,u)^2\!+\!b-\!2\,\varepsilon\,c\,u^2)\,\expt\,[\,\Vert\nabla f(\wh{\bx}^{(k)})\Vert^2\,]\!+\!\tfrac{1}{2}t\,n\,c^2u^2\\
& \phantom{M} +\sqrt{n}\,c\,u\,t\,(L\,t\,(1\!+\!2u)^2\,c\,u\!+\!1\!-\!L\,t\,(1\!+\!2u)^2\!+\!2\,\varepsilon\,u)\,\expt\,[\,\Vert\nabla f(\wh{\bx}^{(k)})\Vert\,].
\end{align*}
The property $c\,u\le a<1$ implies $L\,t\,(1+2u)^2\,c\,u+1-L\,t\,(1+2u)^2+\,2\,\varepsilon\,u\le1+\,2\,\varepsilon\,u$, which indicates that $\expt\,[\,f(\bz^{(k+1)})\,]-\expt\,[\,f(\wh{\bx}^{(k)})\,]$ is bounded from above by
\begin{align*}
\!-(\tfrac{1}{2}-\tfrac{1}{2}\,c\,u-\,2\,\varepsilon\,\,c\,u^2+b)\,t&\,\expt\,[\,\Vert\nabla f(\wh{\bx}^{(k)})\Vert^2\,]+\tfrac{1}{2}\,t\,n\,c^2u^2\nonumber\\
&+\sqrt{n}\,c\,u\,t\,(1+\,2\,\varepsilon\,u)\,\expt\,[\,\Vert\nabla f(\wh{\bx}^{(k)})\Vert\,].
\end{align*}
Further, the property $u\le \frac{a}{c\,+\,4a\,+\,4}$ provides that
\begin{align*}
\tfrac{1}{2}-\tfrac{1}{2}\,c\,u-\,2\,\varepsilon\,c\,u^2+b&=\tfrac{1}{2}+b-\tfrac{1}{2}\,c\,u\,(1+4\,\varepsilon\,u)\\
&\ge \tfrac{1}{2}+b-\tfrac{1}{2}\tfrac{a\,(c+4\,a\,\varepsilon)}{c\,+\,4a\,+\,4}\ge \tfrac{1}{2}(1-a)+b,
\end{align*}
 and property \cref{eq:normf_srh1} implies that \begin{align*}\sqrt{n}\,t\,c\,u\,(1+\,2\,\varepsilon\,u)\,\expt\,[\,\Vert\nabla f(\wh{\bx}^{(k)})\Vert\,]+\tfrac{1}{2}\,t\,n\,c^2u^2\le \tfrac{a}{2}\,t\,\expt\,[\,\Vert\nabla f(\wh{\bx}^{(k)})\Vert^2\,].\end{align*}
As a result, we find $-\tfrac{1}{2}(1+2 b-2a)\,t\,\expt\,[\,\Vert\nabla f(\wh{\bx}^{(k)})\Vert^2\,]\ge\expt\,[\,f(\bz^{(k+1)})\,]-\expt\,[\,f(\wh{\bx}^{(k)})\,]$. Following an analogous argument to that used to obtain \cref{eq:expectedbound_csr1}, we are able to achieve \cref{eq:expectedbound_srh}.

 \textbf{Part (ii)}: In light of \cref{eq:towerprop,eq:sigma_1_c1}, we have $\expt\,[\,\delta_{2,i}^{(k)}\,\sigma_{1,i}^{(k)} \, \nabla f(\wh{\bx}^{(k)})_i\,]=0.$ Proceeding similarly as for Part (i) and \cref{eq:omegak}, we achieve the upper bound for $\expt\,[\,f(\bz^{(k+1)})\,]-\expt\,[\,f(\wh{\bx}^{(k)})\,]$:
\begin{align*} 
& -t\,(\tfrac{1}{2}+b)\,\expt\,[\,\Vert\nabla f(\wh{\bx}^{(k)})\Vert^2\,]+\tfrac{1}{2}\,t\,a^2\,\expt\,[\,\Vert\nabla f(\wh{\bx}^{(k)})\Vert^2\,]+t\,a\,\sqrt{n}\,c\,u\,\expt\,[\,\Vert\nabla f(\wh{\bx}^{(k)})\Vert\,]+\tfrac{1}{2}\,t\,n\,c^2\,u^2.
\end{align*}
Again, property \cref{eq:normf_srh2} implies that  $\tfrac{1}{2}\,a^2\,\expt\,[\,\Vert\nabla f(\wh{\bx}^{(k)})\Vert^2\,]\ge a\,\sqrt{n}\,c\,u\,\expt\,[\,\Vert\nabla f(\wh{\bx}^{(k)})\Vert\,]+\tfrac{1}{2}\,n\,c^2\,u^2.$ Using the similar procedure as for \cref{eq:expectedbound_csr1}, we are able to attain \cref{eq:srh_result2}.
\end{proof}
Comparing \cref{thm:conv-exp_csr} and \cref{coro:conv-exp_srh}, stricter convergence bounds are achieved in \cref{coro:conv-exp_srh} than in \cref{thm:conv-exp_csr} for both conditions \cref{eq:sigma_1_c2,eq:sigma_1_c1}. The employment of $\srh$ in evaluating $\bdelta_2$, particularly, has a positive effect on the convergence speed. Although the impact may be small, a larger value of $\varepsilon$ or $u$ results in a tighter bound on the convergence rate. In the next subsection, we prove that the use of $\ssrh$ is particularly beneficial to the problems suffering from stagnation of GD, where a faster convergence may be obtained.

\subsection{On stochastic rounding: Scenario 2 (stagnation)} \label{sec:s3}
In this subsection, we show that under condition \cref{eq:conditon_2} with stochastic rounding, GD can still update with respect to the rounding errors until a certain level of accuracy. We demonstrate that for the problems suffering from GD stagnation, $\ssrh$ is a better rounding choice than $\csr$, which may lead to a faster convergence.

Condition \cref{eq:conditon_2} implies the inequality $t\,\vert(\nabla f(\wh{\bx}^{(k)})_i+\sigma_{1,i}^{(k)})\,h_{2,i}^{(k)}\vert \leq u\,\vert\wh{x}_i^{(k)} \vert$ and in turn leads to $u\,\Vert \wh{\bx}^{(k)} \Vert \geq t\,\Vert (\nabla f(\wh{\bx}^{(k)})+\bsigma_1^{(k)}) \circ \bh_2^{(k)} \Vert$. This inequality states that the rounding errors in \cref{eq:gd_floatp1} are less important than those in \cref{eq:gd_floatp2} because they are so small that only their signs affect the update. More precisely, under condition \cref{eq:conditon_2}, the magnitudes of the GD updates are constrained to $u$ and $\wh{\bx}^{(k)}$, and we can rewrite \cref{eq:gd_floatp2} as
\begin{align}
\wh{\bx}^{(k+1)}=\,&\mathrm{fl}(\bz^{(k+1)})= \wh{\bx}^{(k)}-\mathbf{d}^{(k)}.\label{eq:gd_underflowSRB}
\end{align}
Depending on the sign of the components of $\wh{\bx}^{(k)}$ and $\nabla f(\wh{\bx}^{(k)})$, the entries of $\mathbf{d}^{(k)}$ can be written as follows,
\begin{subequations} \label{eq:gd_underflow}
\begin{align}
d_i^{(k)}=\,&\begin{cases}
 \wh{x}_i^{(k)}-\mathrm{pr}(\wh{x}_i^{(k)}), \quad &\text{with probability}~p(z_i^{(k+1)}),\\
0, \quad &\text{with probability}~1-p(z_i^{(k+1)}),
\end{cases} \label{eq:gd_underflow1}
\end{align}
or
\begin{align}
d_i^{(k)}=\,&\begin{cases}
0, \quad &\text{with probability}~p(z_i^{(k+1)}),\\
 \wh{x}_i^{(k)}- \mathrm{su}(\wh{x}_i^{(k)}), \quad &\text{with probability}~1-p(z_i^{(k+1)}),
\end{cases}
\label{eq:gd_underflow2}
\end{align}
\end{subequations}
where $p \in \{p_0,\, p_{\varepsilon}, \, \wh p_{\varepsilon}\}$ identifies the stochastic rounding scheme employed ($\csr$, $\srh$ and signed-$\srh$). Note that, the nonzero entries of $\mathbf{d}^{(k)}$ have the same signs as the corresponding entries of $\nabla f(\wh{\mathbf x}^{(k)})+\bsigma_1^{(k)}$. For this reason, in this subsection we do not specify the rounding strategy for \cref{eq:gd_floatp1} and we focus on the effect of the various stochastic rounding methods for evaluating \cref{eq:gd_underflowSRB}. In particular, \cref{eq:gd_underflowSRB} performs similarly to the sign gradient descent method \cite{moulay2019properties} with adaptive stepsize, rather than to the one with fixed stepsize; see, e.g., \cref{eq:gd}. \cref{tab:rounding} shows the four cases that tune the rounding errors in a descent direction for the two updating rules in \cref{eq:gd_underflow}.
Note that we always have $\sign(\wh{x}_i^{(k)})\,\sign(\wh{x}_i^{(k+1)})\ge 0$; this is natural in this scenario as $t\,(\nabla f(\wh{\bx}^{(k)})_i+\sigma_{1,i}^{(k)})\,h_{2,i}^{(k)}$ is relatively small with respect to $\wh{x}_i^{(k)}$. 

\begin{table}[htb!]
{\footnotesize
\caption{The $i$th entry of the updating vector at the $k$th iteration step $d_i^{(k)}$, under different conditions and its corresponding rounding method that tunes it into a descent direction.}\label{tab:rounding}
\begin{center}
\begin{tabular}{llcll}
\cline{1-5} \rule{0pt}{2.7ex}%
 & Sign & $d_i^{(k)}$ & Method & Case\\ \cline{1-5}\rule{0pt}{2.8ex}%
\multirow{2}{*}{$\sign(\wh{x}_i^{(k)})\ \sign(\nabla f(\wh{\bx}^{(k)})_i\!+\!\sigma_{1,i}^{(k)})>0$} &$\sign(\wh{x}_i^{(k)})>0$ &\cref{eq:gd_underflow1} &Round down& I\\[0.5mm]
&$\sign(\wh{x}_i^{(k)})<0$ &\cref{eq:gd_underflow2} &Round up &II\\[0.5mm]
\multirow{2}{*}{$\sign(\wh{x}_i^{(k)})\ \sign(\nabla f(\wh{\bx}^{(k)})_i\!+\!\sigma_{1,i}^{(k)})<0$} & $\sign(\wh{x}_i^{(k)})>0$&\cref{eq:gd_underflow2} &Round up&III\\[0.5mm]
&$\sign(\wh{x}_i^{(k)})<0$ &\cref{eq:gd_underflow1} &Round down&IV\\
\cline{1-5}
\end{tabular}
\end{center}
}
\end{table}

 Since $\mathbf{d}^{(k)}$ is stochastic, we start our analysis by studying the expectation of the updating direction for the rounding method $\csr$. Then, we consider $\srh$ and $\ssrh$ in evaluating \cref{eq:gd_underflowSRB}, and we show that $\ssrh$ always provides a rounding bias in a descent direction if the following condition holds.
 Given a starting vector $\bx^{(0)}$ and an iteration number $k$, we
 denote by $\mathcal{G}_k$ the finite set made of all the possible vectors $\wh{\bx}^{(k)}$ obtained using \cref{eq:gd_underflowSRB}.
\begin{Assumption} \label{assm:srh_cond_stag}
We assume that for all $1 \le i \le n$, and a certain $k >0$, it holds that 
\begin{align*}
 \expt\,[\,\sign(\,\nabla f(\wh{\bx}^{(k)})_i+\sigma_{1,i}^{(k)}\,)\,\,\sign(\nabla f(\wh{\bx}^{(k)})_i)\ \big\vert\ \wh{\bx}^{(k)}\in \mathcal{G}_k\,]> 0.
 \end{align*}
 \end{Assumption}
This assumption means that the evaluated gradient has a larger probability for having the same sign as the exact gradient than the opposite sign. In general, a smaller $u$ indicates a larger probability that \cref{assm:srh_cond_stag} holds. We note that this assumption is necessary for signed-$\srh$, but is not necessary for $\csr$.

\subsubsection{Employment of $\csr$}
Now let us study \cref{eq:gd_underflowSRB} with stochastic rounding. When $\csr$ is applied in \cref{eq:gd_underflowSRB}, the unbiased property implies that $ \expt\,[\,d_i^{(k)} \ \big\vert \ t\,(\sigma_{1,i}^{(k)}+\nabla f(\wh{\bx}^{(k)})_i)\,h_{2,i}^{(k)}\,]=t\,(\sigma_{1,i}^{(k)}+\nabla f(\wh{\bx}^{(k)})_i)\,h_{2,i}^{(k)}$. Based on this, we obtain the following expectation.
\begin{Lemma}\label{lem:expecd_csr}
Under \cref{eq:conditon_2}, when \cref{eq:gd_underflowSRB} is evaluated using $\csr$, we have
\begin{equation} \label{eq:expecd_csr}
\expt\,[\,\nabla f(\wh{\bx}^{(k)})^T\mathbf{d}^{(k)}\,]=\expt\,[\,t\,\nabla f(\wh{\bx}^{(k)})^T\,((\nabla f(\wh{\bx}^{(k)})+\bsigma_1^{(k)}) \circ \bh_2^{(k)}\,)\,].
\end{equation}
\end{Lemma}
\begin{proof}
We are going to use an analogous argument to the one applied to obtain \cref{eq:sigma2_csr}. Based on \cref{eq:towerprop}, we have 
\begin{align*} 
&\expt\,[\,d_i^{(k)} \, \nabla f(\wh{\bx}^{(k)})_i\,]=\expt\,[\,\expt\,[\,d_i^{(k)} \, \nabla f(\wh{\bx}^{(k)})_i\ \big \vert \ \nabla f(\wh{\bx}^{(k)})_i\,]\,]\\
&=\expt\,[\,\expt\,[\,\expt\,[\,d_i^{(k)} \, \nabla f(\wh{\bx}^{(k)})_i\ \big \vert \ t\,(\nabla f(\wh{\bx}^{(k)})_i+\sigma_{1,i}^{(k)})\,h_{2,i}^{(k)},\nabla f(\wh{\bx}^{(k)})_i\,]\ \big \vert \ \nabla f(\wh{\bx}^{(k)})_i\,]\,]\\
&=\expt\,[\,t\,(\nabla f(\wh{\bx}^{(k)})_i+\sigma_{1,i}^{(k)})\,h_{2,i}^{(k)} \, \nabla f(\wh{\bx}^{(k)})_i\,].
\end{align*}
Therefore, we obtain $\expt\,[\,\nabla f(\wh{\bx}^{(k)})^T\mathbf{d}^{(k)}\,]
=\expt\,[\,t\,\nabla f(\wh{\bx}^{(k)})^T\,((\,\nabla f(\wh{\bx}^{(k)})+\bsigma_1^{(k)}) \circ \bh_2^{(k)}\,)\,].$
\end{proof}
On the basis of \cref{lem:expecd_csr}, we may propose an upper bound for $u$ that guarantees the average monotonicity of GD when $\csr$ is applied for both \cref{eq:gd_floatp1} and \cref{eq:gd_underflowSRB}. 

\begin{Proposition} \label{the:ucsr}
Under \cref{model:convexfun} and condition \cref{eq:conditon_2}, suppose that both \cref{eq:gd_floatp1} and \cref{eq:gd_underflowSRB} are computed using $\csr$, $t\le \tfrac{1}{L\,(1+2u)^2}$, and $c\,u<1$. 
 
 (i) Under condition \cref{eq:sigma_1_c2}, if for a $k>0$ it holds that 
\begin{align} \label{eq:normf_csr1_s3}
\expt\,[\,\Vert \nabla f(\wh{\bx}^{(k-1)})\Vert\,]\ge \frac{c\,u\,\sqrt{n}}{1-c\,u}+\frac{u}{t}\sqrt{\frac{1}{1-c\,u}}\sqrt{\expt\,[\,\Vert \wh{\bx}^{(k-1)} \Vert^2\,]},
\end{align} 
 then $\expt\,[\,f(\wh{\bx}^{(k)})\,]\le \expt\,[\,f(\wh{\bx}^{(k-1)})\,]$.
 
(ii) Under condition \cref{eq:sigma_1_c1}, if for a $k>0$ it holds that 
\begin{align} \label{eq:normf_csr2_s3}
\expt\,[\,\Vert \nabla f(\wh{\bx}^{(k-1)})\Vert\,]\ge \frac{u}{t}\sqrt{\expt\,[\,\Vert \wh{\bx}^{(k-1)} \Vert^2\,]}
\end{align} 
then $\expt\,[\,f(\wh{\bx}^{(k)})\,]\le \expt\,[\,f(\wh{\bx}^{(k-1)})\,]$.
\end{Proposition}
\begin{proof}
The property of $\csr$, i.e., $ \expt\,[\,d_i^{(k)} \vert\,t\,(\nabla f(\wh{\bx}^{(k)})_i+\sigma_{1,i}^{(k)})\,h_{2,i}^{(k)}\,]=t\,(\nabla f(\wh{\bx}^{(k)})_i+\sigma_{1,i}^{(k)})\,h_{2,i}^{(k)}$, indicates that 
\begin{align}
t\,(\nabla f(\wh{\bx}^{(k)})_i+\sigma_{1,i}^{(k)})\,h_{2,i}^{(k)} =\begin{cases}
(\,\wh{x}_i^{(k)}-\mathrm{pr}(\wh{x}_i^{(k)})\,)\,p_0(z_i^{(k+1)}), \quad&\text{for \cref{eq:gd_underflow1},}\\
(\,\wh{x}_i^{(k)}-\mathrm{su}(\wh{x}_i^{(k)})\,)\,(1-p_0(z_i^{(k+1)})\,)\quad&\text{for \cref{eq:gd_underflow2}.}\end{cases}\nonumber
\end{align}
Therefore, we have
\begin{align} \label{eq:d^2_csr}
\expt\,[\,(d_i^{(k)})^2\ &\big\vert\ t\,(\nabla f(\wh{\bx}^{(k)})_i+\sigma_{1,i}^{(k)})\,h_{2,i}^{(k)}\,]\nonumber\\
&=\begin{cases}
(\,\wh{x}_i^{(k)}-\mathrm{pr}(\wh{x}_i^{(k)})\,)^2\,p_0(z_i^{(k+1)}), \quad&\text{for \cref{eq:gd_underflow1}}\\
(\,\wh{x}_i^{(k)}-\mathrm{su}(\wh{x}_i^{(k)})\,)^2\,(1-p_0(z_i^{(k+1)})\,)\quad&\text{for \cref{eq:gd_underflow2}} 
\end{cases}\nonumber\\
&\underset{\!\!\!\!\!\!\!\!\text{\cref{eq:conditon_2}}}{\le}\begin{cases}\frac{1}{2}\,\vert \wh{x}_i^{(k)}-\mathrm{pr}(\wh{x}_i^{(k)})\vert^2 \quad&\text{for \cref{eq:gd_underflow1}} \\
\frac{1}{2}\,\vert \wh{x}_i^{(k)}-\mathrm{su}(\wh{x}_i^{(k)})\vert^2 \quad&\text{for \cref{eq:gd_underflow2}}
\end{cases}\nonumber\\
&\le 2u^2\,(\wh{x}_i^{(k)})^2.
\end{align}
As a result, we obtain 
\begin{align*}
 \expt\,[\,\Vert \mathbf{d}^{(k)} \Vert^2\,] &=\sum_{i=1}^{n} \expt\,[\,(d_i^{(k)})^2\,] =\sum_{i=1}^{n} \expt\,[\, \expt\,[\,(d_i^{(k)})^2\ \big\vert\ t\,(\nabla f(\wh{\bx}^{(k)})_i+\sigma_{1,i}^{(k)})\,h_{2,i}^{(k)}\,]\,]\nonumber\\
 &\le \,2u^2\,\sum_{i=1}^{n}\expt\,[\,(\wh{x}_i^{(k)})^2\,]=2u^2\,\expt\,[\,\Vert \wh{\bx}^{(k)} \Vert^2 \,].
\end{align*}

{\bf Part (i)}: Using \cref{lem:expecd_csr}, on the basis of \cref{eq:lc_ineq}, we have 
\begin{align}
&\expt\,[\,f(\wh{\bx}^{(k+1)})\,]\le\expt\,[\,f(\wh{\bx}^{(k)})\,]-\expt\,[\,\nabla f(\wh{\bx}^{(k)})^T\mathbf{d}^{(k)}\,]+u^2\,L\,\expt\,[\,\Vert \wh{\bx}^{(k)} \Vert^2\,]\nonumber\\
&=\expt\,[\,f(\wh{\bx}^{(k)})\,]-t\,\expt\,[\,\nabla f(\wh{\bx}^{(k)})^T(\,(\nabla f(\wh{\bx}^{(k)})+\bsigma_1^{(k)}) \circ \bh_2^{(k)}\,)\,]+u^2\,L\,\expt\,[\,\Vert \wh{\bx}^{(k)} \Vert^2\,]\nonumber.
\end{align}
We mimic the proof of \cref{eq:expectedfxk_csr} and obtain an upper bound for $\expt\,[\,f(\wh{\bx}^{(k+1)})\,]-\expt\,[\,f(\wh{\bx}^{(k)})\,]$:
\begin{align*}
&-t\,\expt\,[\,\Vert \nabla f(\wh{\bx}^{(k)})\Vert^2\,]-t\,\expt\,[\,\nabla f(\wh{\bx}^{(k)})^T\bsigma_1^{(k)}\,]+u^2\,L\,\expt\,[\,\Vert \wh{\bx}^{(k)} \Vert^2\,]\\
&\underset{\text{\cref{eq:sigma_1_c2}}}{\le}-t\,(1-c\,u)\,\expt\,[\,\Vert \nabla f(\wh{\bx}^{(k)})\Vert^2\,]+ t\,c\,u\,\sqrt{n}\,\expt\,[\,\Vert\nabla f(\wh{\bx}^{(k)})\Vert\,]+u^2\,L\,\expt\,[\,\Vert \wh{\bx}^{(k)} \Vert^2\,].
\end{align*} 
On the basis of Jensen's inequality, properties \cref{eq:normf_csr1_s3} and $t\le \tfrac{1}{L\,(1+2u)^2}<\tfrac{1}{L}$ we have that 
\begin{align*}
 t\,(1-c\,u)\,\expt\,[\,\Vert \nabla f(\wh{\bx}^{(k)})\Vert^2\,]\ge t\,c\,u\,\sqrt{n}\,\expt\,[\,\Vert\nabla f(\wh{\bx}^{(k)})\Vert\,]+u^2\,L\,\expt\,[\,\Vert \wh{\bx}^{(k)} \Vert^2\,].
\end{align*}
This implies that $\expt\,[\,f(\wh{\bx}^{(k+1)})\,]\le \expt\,[\,f(\wh{\bx}^{(k)})\,]$, which concludes the proof of claim (i).

\textbf{Part (ii)}: Replicating the proof of claim (i), we obtain \[\expt\,[\,f(\wh{\bx}^{(k+1)})\,]-\expt\,[\,f(\wh{\bx}^{(k)})\,]\le -t\,\expt\,[\,\Vert \nabla f(\wh{\bx}^{(k)})\Vert^2\,]+u^2\,L\,\expt\,[\,\Vert \wh{\bx}^{(k)} \Vert^2\,].\]
Condition \cref{eq:normf_csr2_s3} indicates that $t\,\expt\,[\,\Vert \nabla f(\wh{\bx}^{(k)})\Vert\,]^2 \ge u^2\,L\,\expt\,[\,\Vert \wh{\bx}^{(k)} \Vert]^2$, which implies $\expt\,[\,f(\wh{\bx}^{(k)})\,]\ge \expt\,[\,f(\wh{\bx}^{(k+1)})\,]$, concluding the proof.
\end{proof}
\subsubsection{Employment of $\ssrh$}\label{sec:directionofroundingerrors}
We now shed some light on why the use of $\srh$ in this context may be problematic and why it is favorable to consider $\ssrh$. Let us assume that $\srh$ is used for evaluating the updating rule for Case I in \cref{tab:rounding}; then, by means of the relation $p_{\varepsilon}(x)= \varphi(p_0(x)-\sign (x)\varepsilon)$, we have
\begin{align} \label{eq:dsrh}
&\expt\,[\,d_i^{(k)}\ \big \vert \  t\,(\nabla f(\wh{\bx}^{(k)})_i+\sigma_{1,i}^{(k)})\,h_{2,i}^{(k)}\,] =\, (\wh{x}_i^{(k)}-\mathrm{pr}(\wh{x}_i^{(k)}))\,p_{\varepsilon}(z_i^{(k+1)})\\
 & \phantom{MM}=\,\begin{cases}
 \wh{x}_i^{(k)}-\mathrm{pr}(\wh{x}_i^{(k)}),& p_{\varepsilon}=1.\\[1mm]
t\,(\nabla f(\wh{\bx}^{(k)})_i+\sigma_{1,i}^{(k)})\,h_{2,i}^{(k)}-
(\wh{x}_i^{(k)}-\mathrm{pr}(\wh{x}_i^{(k)}))\, \,\sign(\wh{x}_i^{(k)})\,\varepsilon,\quad &\text{otherwise}.\nonumber
\end{cases}
\end{align}
From \cref{eq:dsrh}, it can be seen that when $p_{\varepsilon}\ne 1$, it is hard to control the updating direction of GD, unless $\wh{x}_i^{(k)}$ has always the opposite sign of $\nabla f(\wh{\bx}^{(k)})_i$. Clearly, this cannot be guaranteed by $\srh$ but can be easily achieved by $\ssrh$ by using $\nabla f(\wh{\bx}^{(k)})_i+\sigma_{1,i}^{(k)}$ instead of $v$ for the corresponding input $x_i^{(k)}$ in \cref{eq:ssrh-experror}. Although we cannot guarantee that $\nabla f(\wh{\bx}^{(k)})_i+\sigma_{1,i}^{(k)}$ always has the same sign as $\nabla f(\wh{\bx}^{(k)})_i$, with this choice, we will show that, on average, $\ssrh$ may achieve a faster convergence than $\csr$ if \cref{assm:srh_cond_stag} holds.
\begin{Lemma} \label{lem:expecd_srhsigned}
Under condition \cref{eq:conditon_2} and \cref{assm:srh_cond_stag}, when \cref{eq:gd_underflowSRB} is computed using $\ssrh$, if $\varepsilon\le0.5$ holds, then we have
\begin{equation} \label{eq:expecd_srhsigned}
\expt\,[\,\nabla f(\wh{\bx}^{(k)})^T\mathbf{d}^{(k)}\,]> \expt\,[\,t\,\nabla f(\wh{\bx}^{(k)})^T\,((\nabla f(\wh{\bx}^{(k)})+\bsigma_1^{(k)}) \circ \bh_2^{(k)}\,)\,].
\end{equation}
\end{Lemma}
\begin{proof}
When $\varepsilon\le0.5$, condition \cref{eq:conditon_2} shows that $t\,(\nabla f(\wh{\bx}^{(k)})_{i}+\sigma_{1,i}^{(k)})\,h_{2,i}^{(k)}+\varepsilon\,(\,\wh{x}_i^{(k)}-\mathrm{pr}(\wh{x}_i^{(k)})\,) \le \wh{x}_i^{(k)}-\mathrm{pr}(\wh{x}_i^{(k)})$ or $t\,(\nabla f(\wh{\bx}^{(k)})_{i}+\sigma_{1,i}^{(k)})\,h_{2,i}^{(k)}+\varepsilon\,(\,\mathrm{su}(\wh{x}_i^{(k)})-\wh{x}_i^{(k)}\,) \le \mathrm{su}(\wh{x}_i^{(k)})-\wh{x}_i^{(k)}.$
Together with \cref{def:srh} (cf.~\cref{eq:epsilon}), it indicates that $0<\wh{p}_{\varepsilon}<1$.

Let us denote by $\mathcal{S}_1$ the finite set of values that can be assumed for the $i$th component of $t\,(\nabla f(\wh{\bx}^{(k)})+\bsigma_1^{(k)}) \circ \bh_2^{(k)}$ and that satisfy Case I in \cref{tab:rounding}. Analogously we define $\mathcal{S}_j$, $j=2,3,4$, for Cases II, III, and IV.
When $t\,(\nabla f(\wh{\bx}^{(k)})_i+\sigma_{1,i}^{(k)})\,h_{2,i}^{(k)}\in \mathcal{S}_1\cup\mathcal{S}_4$, in particular we have $\nabla f(\wh{\bx}^{(k)})_{i}+\sigma_{1,i}^{(k)}>0$. Taking the conditional expectation of \cref{eq:gd_underflow1} and proceeding analogously to \cref{eq:dsrh}, on the basis of the property that $\wh{p}_{\varepsilon}\ne 0$ or $1$, we obtain
\begin{align} \label{eq:case1}
\expt\,[&\,d_i^{(k)}\ \big\vert\ \wh{\bx}^{(k)},\, t\,(\nabla f(\wh{\bx}^{(k)})_i+\sigma_{1,i}^{(k)})\,h_{2,i}^{(k)}\in\mathcal{S}_1\cup\mathcal{S}_4\,]\\
 & \phantom{MM}=\,(\wh{x}_i^{(k)}-\mathrm{pr}(\wh{x}_i^{(k)}))\,\varphi(\,p_0(z_i^{(k+1)})+\sign(\nabla f(\wh{\bx}^{(k)})_i+\sigma_{1,i}^{(k)})\, \varepsilon\,)\nonumber\\
&\phantom{MM}=\,
 t\,(\nabla f(\wh{\bx}^{(k)})_i+\sigma_{1,i}^{(k)})\,h_{2,i}^{(k)}+
(\wh{x}_i^{(k)}-\mathrm{pr}(\wh{x}_i^{(k)}))\,\sign(\nabla f(\wh{\bx}^{(k)})_i+\sigma_{1,i}^{(k)})\,\varepsilon.\nonumber
\end{align}
For $\nabla f(\wh{\bx}^{(k)})_{i}+\sigma_{1,i}^{(k)}<0$ (Cases II and III), applying the similar steps as for \cref{eq:case1}, we obtain
\begin{align} \label{eq:case2}
 &\expt\,[\,d_i^{(k)}\ \big\vert \ \wh{\bx}^{(k)},\,t\,(\nabla f(\wh{\bx}^{(k)})_i+\sigma_{1,i}^{(k)})\,h_{2,i}^{(k)}\in \mathcal{S}_2\cup\mathcal{S}_3\,]\\
&\phantom{MM}=\,
 t\,(\nabla f(\wh{\bx}^{(k)})_i+\sigma_{1,i}^{(k)})\,h_{2,i}^{(k)}+ \, 
\vert \wh{x}_i^{(k)}-\mathrm{su}(\wh{x}_i^{(k)})\vert\,\sign(\nabla f(\wh{\mathbf x}^{(k)})_i+\sigma_{1,i}^{(k)})\, \varepsilon.\nonumber
\end{align}
Let us define the function $g: \Real^n\to \Real^+$ as
\begin{align}\label{eq:g}
g(\wh{\bx}^{(k)})&:=\varepsilon\,\vert \nabla f(\wh{\bx}^{(k)})_i\vert\,(\wh{x}_i^{(k)}\!-\!\mathrm{pr}(\wh{x}_i^{(k)}))\,P(\,t\,(\nabla f(\wh{\bx}^{(k)})_i\!+\!\sigma_{1,i}^{(k)})\,h_{2,i}^{(k)}\in \mathcal{S}_1\cup\mathcal{S}_4\,)\nonumber\\&\ +\varepsilon\,\vert \nabla f(\wh{\bx}^{(k)})_i\vert\,(\mathrm{su}(\wh{x}_i^{(k)})\!-\!\wh{x}_i^{(k)})\,P(\,t\,(\nabla f(\wh{\bx}^{(k)})_i\!+\!\sigma_{1,i}^{(k)})\,h_{2,i}^{(k)}\in \mathcal{S}_2\cup\mathcal{S}_3\,).
\end{align}
Then we obtain
\begin{align*}
 &\expt\,[\,d_i^{(k)} \, \nabla f(\wh{\bx}^{(k)})_i\ \big\vert\ \wh{\bx}^{(k)},\, t\,(\nabla f(\wh{\bx}^{(k)})_i+\sigma_{1,i}^{(k)})\,h_{2,i}^{(k)}\,]\\
&\phantom{m}=
 t\,(\nabla f(\wh{\bx}^{(k)})_i+\sigma_{1,i}^{(k)})\,h_{2,i}^{(k)}+g(\wh{\bx}^{(k)})\, \sign(\nabla f(\wh{\mathbf x}^{(k)})_i+\sigma_{1,i}^{(k)})\, \sign(\nabla f(\wh{\mathbf x}^{(k)})_i).
\end{align*}
On the basis of \cref{eq:towerprop}, it is further proven that 
\begin{align*}
\expt\,[\,d_i^{(k)}& \, \nabla f(\wh{\bx}^{(k)})_i\ \big\vert\ \wh{\bx}^{(k)}\,]\\
&=\expt\,[\,\expt\,[\,d_i^{(k)} \, \nabla f(\wh{\bx}^{(k)})_i\ \big\vert\ \wh{\bx}^{(k)},\, t\,(\nabla f(\wh{\bx}^{(k)})_i+\sigma_{1,i}^{(k)})\,h_{2,i}^{(k)}\,]\ \big\vert\ \wh{\bx}^{(k)}\,]\\
&=\expt\,[\,t\,(\nabla f(\wh{\bx}^{(k)})_i+\sigma_{1,i}^{(k)})\,h_{2,i}^{(k)}\,\nabla f(\wh{\bx}^{(k)})_i\ \big\vert\ \wh{\bx}^{(k)}\,]\\
 &\phantom{MM}\quad+\expt\,[\,g(\wh{\bx}^{(k)})\,
 \sign(\nabla f(\wh{\mathbf x}^{(k)})_i+\sigma_{1,i}^{(k)})\,\sign(\nabla f(\wh{\mathbf x}^{(k)})_i)\ \big\vert\ \wh{\bx}^{(k)}\,]\\
 &=\,
 \expt\,[\,t\,(\nabla f(\wh{\bx}^{(k)})_i+\sigma_{1,i}^{(k)})\,h_{2,i}^{(k)}\,\nabla f(\wh{\bx}^{(k)})_i\ \big\vert\ \wh{\bx}^{(k)}\,]\\
 &\phantom{MM}\quad+g(\wh{\bx}^{(k)})\,\expt\,[\,
 \sign(\nabla f(\wh{\mathbf x}^{(k)})_i+\sigma_{1,i}^{(k)})\,\sign(\nabla f(\wh{\mathbf x}^{(k)})_i)\ \big\vert\ \wh{\bx}^{(k)}\,].
\end{align*}
Based on the law of total expectation and under \cref{assm:srh_cond_stag}, we achieve
\begin{align}
 &\expt\,[\,d_i^{(k)} \, \nabla f(\wh{\bx}^{(k)})_i\,]=\expt\,[\,t\,(\nabla f(\wh{\bx}^{(k)})_i+\sigma_{1,i}^{(k)})\,h_{2,i}^{(k)}\,\nabla f(\wh{\bx}^{(k)})_i\,] \nonumber\\
&\phantom{MM}+\sum_{\by\in\mathcal{G}_k}g(\by)\,\expt\,[\,
 \sign(\nabla f(\wh{\mathbf x}^{(k)})_i+\sigma_{1,i}^{(k)})\,\sign(\nabla f(\wh{\mathbf x}^{(k)})_i)\ \big\vert\ \wh{\bx}^{(k)}=\by].\nonumber
\end{align}
Let us define\begin{align} \label{eq:Q_i} Q_i^{(k)}:=\sum_{\by\in\mathcal{G}_k}g(\by)\,\expt\,[\,
 \sign(\nabla f(\wh{\mathbf x}^{(k)})_i+\sigma_{1,i}^{(k)})\,\sign(\nabla f(\wh{\mathbf x}^{(k)})_i)\ \big\vert\ \wh{\bx}^{(k)}=\by].
\end{align}
Since $g(\by)>0$, \cref{assm:srh_cond_stag} implies that $Q_i^{(k)}> 0$. Therefore, we have 
\begin{align} \label{eq:expdssrh}
\expt\,[\,\nabla f(\wh{\bx}^{(k)})^T\mathbf{d}^{(k)}\,]=\expt\,[\,t\,\nabla f(\wh{\bx}^{(k)})^T\,((\nabla f(\wh{\bx}^{(k)})+\bsigma_1^{(k)}) \circ \bh_2^{(k)}\,)\,]+\sum_{i=1}^n Q_{i}^{(k)}.\end{align}
\end{proof}
Based on the fact that $\vert \wh{x}_i^{(k)}-\mathrm{pr}(\wh{x}_i^{(k)})\vert \le 2u\,\,\vert\wh{x}_i^{(k)} \vert$ and $\vert \wh{x}_i^{(k)}-\mathrm{su}(\wh{x}_i^{(k)})\vert \le 2u\,\,\vert \wh{x}_i^{(k)} \vert$, one may check that $g(\wh{\bx}^{(k)}) \le 2\,\varepsilon\,u\vert \nabla f(\wh{\bx}^{(k)})_i\,\vert\,\vert \wh{x}_i^{(k)} \vert$ (cf.~\cref{eq:g}). Equation \cref{eq:Q_i} shows that the magnitude of $Q_i^{(k)}$ is determined by $g(\wh{\bx}^{(k)})$. Therefore, based on $g(\wh{\bx}^{(k)}) \le 2\,\varepsilon\,u\vert \nabla f(\wh{\bx}^{(k)})_i\,\vert\,\vert\wh{x}_i^{(k)} \vert$, we have\begin{align} \label{eq:bound_sumQ}
\sum_{i=1}^n Q_{i}^{(k)}\le 2\,\varepsilon\,u\,\,\vert\nabla f(\wh{\bx}^{(k)})\vert^T \, \vert \wh{\bx}^{(k)}|.
\end{align}
In particular, the magnitude of $\sum_{i=1}^n Q_{i}^{(k)}$ may depend on the magnitudes of $\varepsilon$, $u$, $\nabla f(\wh{\bx}^{(k)})$ and $\wh{\bx}^{(k)}$.
In contrast to the constant stepsize of GD with exact arithmetic, $\ssrh$ makes \cref{eq:gd_underflow} perform similarly to GD with adaptive stepsize, where the stepsize is automatically adjusted to its current iterate $\wh{\bx}^{(k)}$.

Under the same conditions as the analysis with $\csr$, we show that $\ssrh$ guarantees the strict monotonicity of the GD iteration.
\begin{Proposition} \label{thm:modifiedsrh}
 Under the same condition as in \cref{the:ucsr} and \cref{assm:srh_cond_stag}, suppose that \cref{eq:gd_floatp1} and \cref{eq:gd_underflowSRB} are computed using $\csr$ and $\ssrh$ with $\varepsilon\le0.5$, respectively. 
 
 (i) Under condition \cref{eq:sigma_1_c2}, if  for a $k>0$ it holds that 
\begin{align} \label{eq:normf_srh1_s3}
\expt\,[\,\Vert \nabla f(\wh{\bx}^{(k-1)})\Vert\,]\ge \frac{c\,u\,\sqrt{n}}{1-c\,u}+\frac{u}{t}\sqrt{\frac{1+2\,\varepsilon}{1-c\,u}}\sqrt{\expt\,[\,\Vert \wh{\bx}^{(k-1)} \Vert^2\,]},
\end{align} 
 then $\expt\,[\,f(\wh{\bx}^{(k)})\,]< \expt\,[\,f(\wh{\bx}^{(k-1)})\,]$.
 
(ii) Under condition \cref{eq:sigma_1_c1}, if for a $k>0$ it holds that 
\begin{align} \label{eq:normf_srh2_s3}
\expt\,[\,\Vert \nabla f(\wh{\bx}^{(k-1)})\Vert\,]\ge \frac{u}{t}\sqrt{1+2\,\varepsilon}\sqrt{\expt\,[\,\Vert \wh{\bx}^{(k-1)} \Vert^2\,]}
\end{align} 
then $\expt\,[\,f(\wh{\bx}^{(k)})\,]< \expt\,[\,f(\wh{\bx}^{(k-1)})\,]$. 
\end{Proposition}
\begin{proof}
Following a similar argument as for \cref{eq:d^2_csr}, when using $\ssrh$ to evaluate \cref{eq:gd_underflowSRB}, we obtain
$\expt\,[\,(d_i^{(k)})^2\vert\,t\,(\nabla f(\wh{\bx}^{(k)})_i+\sigma_{1,i}^{(k)})\,h_{2,i}^{(k)}\,]
\le 2\,(1+2\,\varepsilon)\,u^2\,(\wh{x}_i^{(k)})^2,$
and by means of the law of total expectation we finally get
\begin{align}
 \expt\,[\,\Vert \mathbf{d}^{(k)} \Vert^2\,] \le 2\,(1+2\,\varepsilon)\,u^2\,\expt\,[\,\Vert \wh{\bx}^{(k)} \Vert^2 \,].\nonumber
\end{align}

\textbf{Part (i)}:
On the basis of \cref{eq:lc_ineq} and using \cref{lem:expecd_srhsigned}, \cref{eq:sigma_1_c2,eq:expdssrh}, we have 
\begin{align} \label{eq:inq_ssrh}
\expt\,[\,f(\wh{\bx}^{(k+1)})\,]
\le&\, \expt\,[\,f(\wh{\bx}^{(k)})\,]-t\,(1-c\,u)\,\expt\,[\,\Vert \nabla f(\wh{\bx}^{(k)})\Vert^2\,]-\sum_{i=1}^n Q_{i}^{(k)}\nonumber\\
&\quad+ t\,c\,u\,\sqrt{n}\,\expt\,[\,\Vert\nabla f(\wh{\bx}^{(k)})\Vert\,]+(1+2\,\varepsilon)\,u^2\,L\,\expt\,[\,\Vert \wh{\bx}^{(k)} \Vert^2\,]\nonumber\\
\underset{\text{\cref{eq:normf_srh1_s3}}}{\quad\le}&\,\expt\,[\,f(\wh{\bx}^{(k)})\,]-\sum_{i=1}^n Q_{i}^{(k)},
\end{align}
where $Q_i^{(k)}>0$.

\textbf{Part (ii)}: With the same argument used for (i), and by means of \cref{eq:sigma_1_c1}, we achieve
\begin{align} \label{eq:inq_ssrh1}
\expt\,[\,f(\wh{\bx}^{(k+1)})\,]\le \,&\expt\,[\,f(\wh{\bx}^{(k)})\,]-t\,\expt\,[\,\Vert \nabla f(\wh{\bx}^{(k)})\Vert^2\,]+(1+2\,\varepsilon)\,u^2\,L\,\expt\,[\,\Vert \wh{\bx}^{(k)} \Vert^2\,]-\sum_{i=1}^n Q_{i}^{(k)},\nonumber\\
\underset{\text{\cref{eq:normf_srh2_s3}}}{\quad\le}\,& \expt\,[\,f(\wh{\bx}^{(k)})\,]-\sum_{i=1}^n Q_{i}^{(k)},
\end{align} 
which gives the claim
\end{proof}

Despite the fact that \cref{thm:modifiedsrh} does not provide a significant advantage of $\ssrh$ with respect to $\csr$, inequality \cref{eq:inq_ssrh} suggests that $\ssrh$ may lead to a faster convergence depending on the accumulated rounding bias that is determined by the value of $\sum_{i=1}^n Q_{i}^{(k)}$. According to \cref{eq:bound_sumQ}, the magnitude of $\sum_{i=1}^n Q_{i}^{(k)}$ depends on $\varepsilon$, $u$, and $\wh{\bx}^{(k)}$. For exact arithmetic, the convergence rate of GD is determined by $t$ and $\nabla f(\bx^{(k)})$. When $\bx^{(k)}=\wh{\bx}^{(k)}$ and $t\,\Vert \nabla f(\wh{\bx}^{(k)})\Vert$ is smaller than $u\,\Vert \wh{\bx}^{(k)}\Vert$, then $\ssrh$ may even lead to a faster convergence than exact arithmetic. In particular, $\ssrh$ may be beneficial in solving multi-variable optimization problems, especially for the case when GD stagnates in some coordinates of $\bx$. In the next section, we show, by means of numerical simulations, that this advantage is indeed tangible.

\section{Simulation study} \label{sec:simulation}
In this section, we validate the theoretical analysis by testing the performances of GD for various choices of the rounding schemes used in performing steps \cref{eq:gd_floatp1} and \cref{eq:gd_floatp2}. As case studies, we consider the minimization of quadratic functions, the training of an MLR, and the training of a two-layer NN, with low-precision floating-point computations. 

As representatives of low-precision number formats, we consider {\sf bfloat16} for the quadratic optimization and {\sf binary8} for the training of MLR and NN. The baselines are obtained by {\sf binary32} with the default rounding mode in IEEE, i.e., $\rn$ with ties to even. See \cref{sec:flp} for the complete descriptions of the number formats. Note that the roundoff errors caused by {\sf binary32} are almost negligible compared to the limited precision employed by {\sf bfloat16} and {\sf binary8}. We look at the comparison with the baseline as a comparison with GD in exact arithmetic. Further, all the expectations and variances obtained when using $\csr$, $\srh$, and $\ssrh$ are estimated over 20 simulations.\footnote{The MATLAB code is
available upon request to the corresponding author.} 
Note that all the plots in this section adopt a logarithmic scale along the vertical axis. 

\subsection{Quadratic optimization}
In this first experiment we apply GD to
the quadratic optimization problem $\min_{\bx\in\mathbb R^n} f(\bx)=\tfrac{1}{2}\,(\bx-\bx^*)^T\!A\,(\bx-\bx^*)$,
for two choices of the matrix $A$, the starting vector $\bx^{(0)}$, the minimizer $\bx^*$, and the stepsize $t$. Our first choice (Setting I) is $A=\mathrm{diag}(10^{-3},\dots,10^{-3},1)\in\mathbb R^{1000\times 1000}$, $\bx^{(0)}=[10^{-3},\dots,10^{-3},1]^T$, $\bx^{*}=[0,\dots,0]^T$, and $t=10^{-5}$. 
The stepsize is relatively small compared to $L^{-1}$ and all the entries of the initial point are close to the minimizer apart from the last entry. In the second choice (Setting II), we consider a symmetric matrix $A\in\mathbb R^{1000\times 1000}$ containing only nonzero elements and having eigenvalues $1,\dots, 1000$, $\bx^{(0)}=[1000,999,\dots,1]^T$, $\bx^{*}=[2^{-4},\dots,2^{-4}]^T$, and $t=L^{-1}=10^{-3}$. We remark that in Setting II, we select the largest possible stepsize among those that guarantee convergence, and a starting point that is far from the minimizer. 

\begin{figure*}[tbhp]
\centering
\subfloat[]{\label{fig:qa}\includegraphics[width=0.45\textwidth]{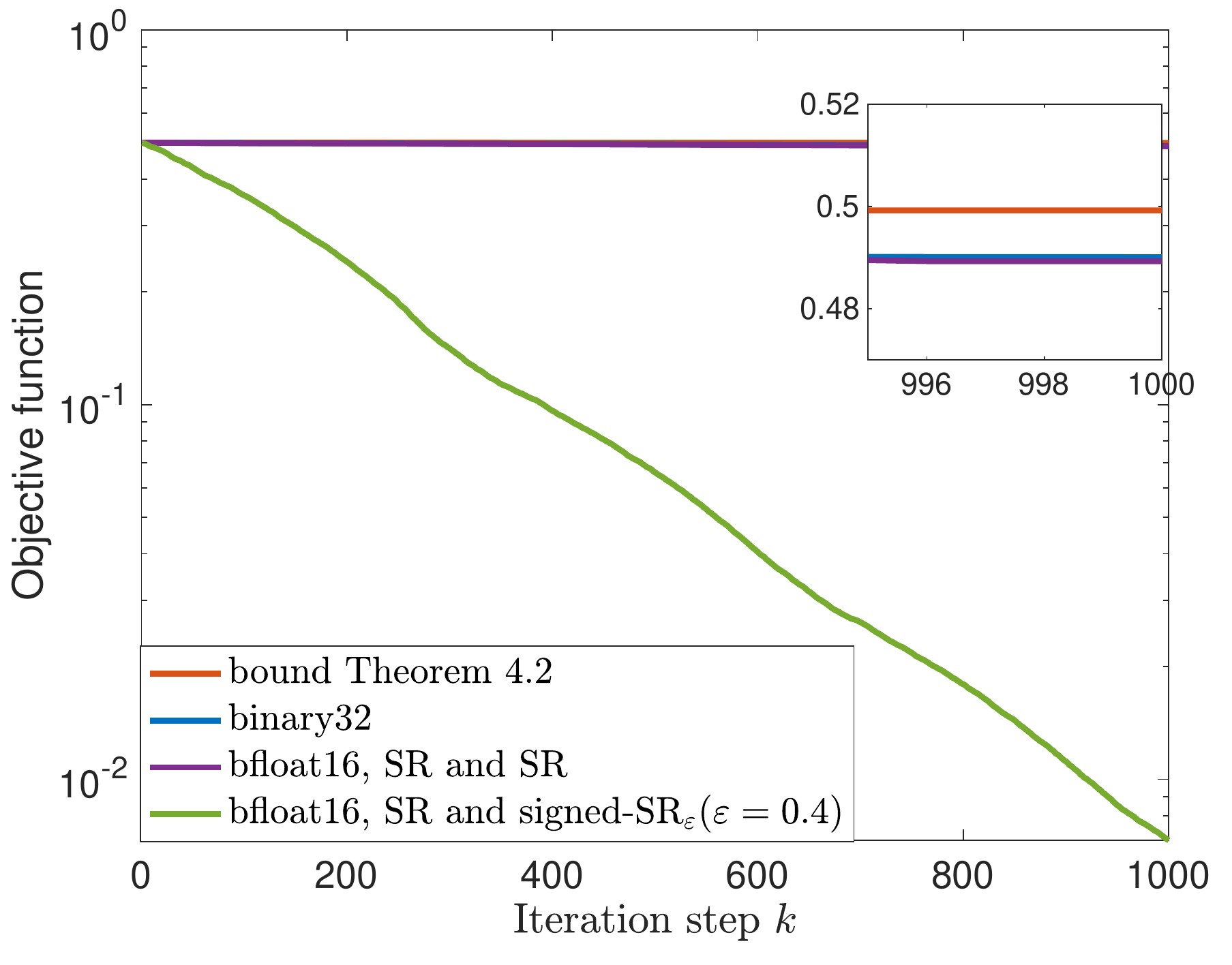}}\quad\quad
\subfloat[]{\label{fig:qb}\includegraphics[width=0.45\textwidth]{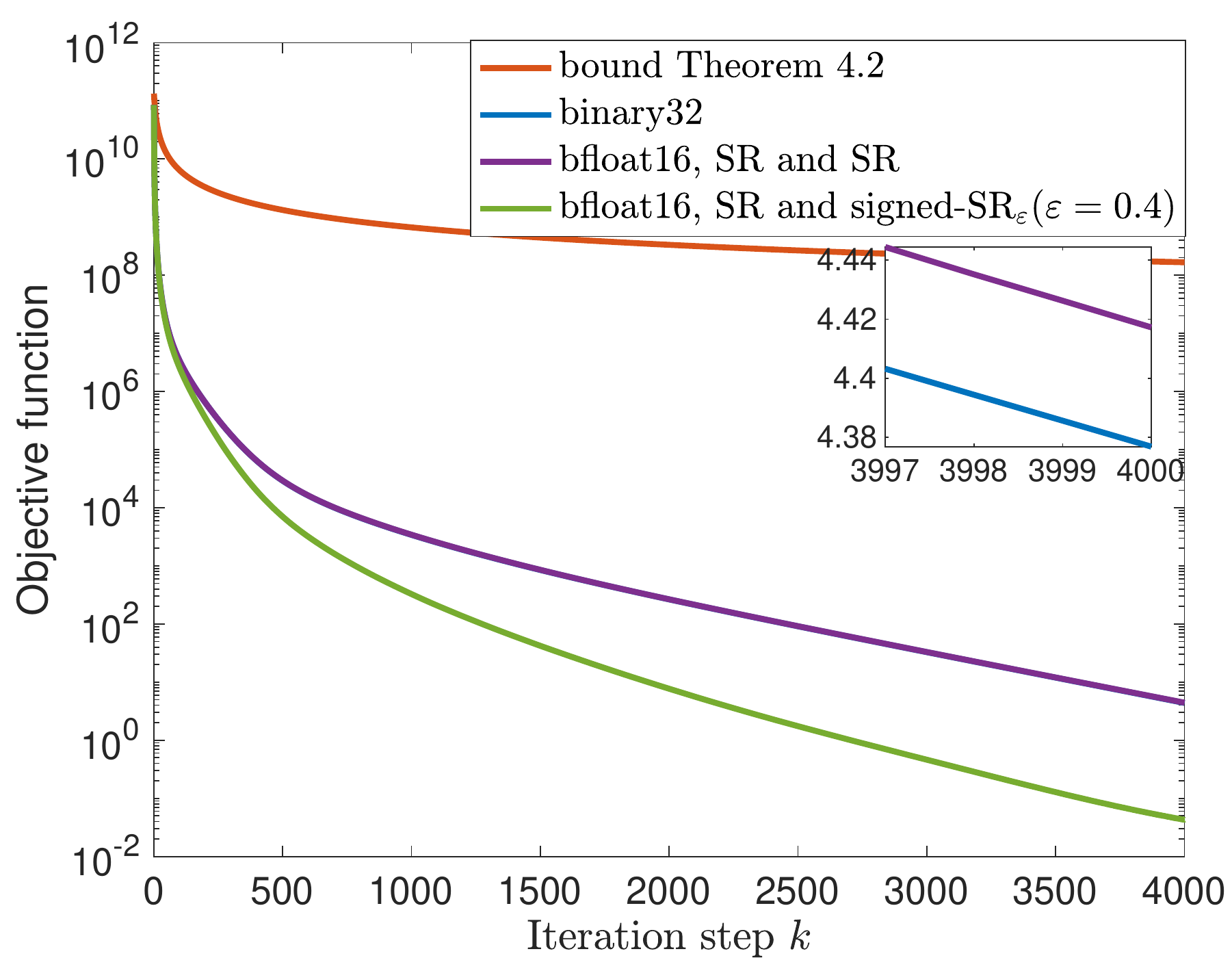}}
\caption{Comparison of the bound from \cref{theorem:convergencerate} and the expectation of the objective function while using $\csr$ to implement \cref{eq:gd_floatp1} and different rounding schemes to implement \cref{eq:gd_floatp2}, i.e., $\csr$ and $\ssrh$ with $\varepsilon= 0.4$ for Setting I (a) and Setting II (b).}\label{fig:quadratic}
\end{figure*}
\cref{fig:quadratic} illustrates the convergence history of the implementations of GD with various choices of number formats and rounding schemes, together with the bound $\frac{2L}{4+L\,t\,k} \, \Vert \bx^{(0)}-\bx^*\Vert^2$ from \cref{theorem:convergencerate}. More precisely, we compare the objective function values obtained using {\sf binary32} and RN with the average of the objective function values obtained using {\sf bfloat16} and $\csr$ for \cref{eq:gd_floatp1} and different stochastic rounding schemes for \cref{eq:gd_floatp2}. The results for Setting I are shown in \cref{fig:qa} while those for Setting II are depicted in \cref{fig:qb}. We did not include the convergence history of GD with RN, in the {\sf bfloat16} format, as it stagnates from the very beginning of the GD iteration. From \cref{fig:qa}, it can be seen that the bound in \cref{theorem:convergencerate} is very close to the objective function obtained by {\sf binary32} and the one by {\sf bfloat16} with $\csr$. Using $\ssrh$ to implement \cref{eq:gd_floatp2}, we achieve almost linear convergence for GD, which is consistent with the discussion after \cref{thm:modifiedsrh}. From \cref{fig:qb}, it can be seen that, in Setting II, the bound in \cref{theorem:convergencerate} is not strict anymore. Again, the employment of $\csr$ with {\sf bfloat16} leads to similar expectations of the objective function values to the one with {\sf binary32}. The utilization of $\ssrh$ yields a much faster convergence than both {\sf binary32} with $\rn$ and {\sf bfloat16} with $\csr$. Additionally, at the $4000$th iteration step, the averaged relative error $\Vert\bx^{(4000)}-\bx^* \Vert/\Vert\bx^* \Vert$ obtained by $\ssrh$ is $0.12$ while that acquired by $\csr$ is $1.50$. We can conclude that, for both settings, the use of $\ssrh$ for \cref{eq:gd_floatp2} accelerates the convergence significantly compared to both $\rn$ and $\csr$.
\subsection{Multinomial logistic regression (MLR)}
MLR is an optimization problem that models multi-label classification tasks. The objective function of MLR is convex \cite{bohning1992multinomial}; for a detailed description; see \cite[pp.~269--272]{hosmer2013applied}. We consider the solution of MLR for classifying the MNIST database \cite{deng2012mnist}, which is a large database of 10 handwritten digits (from 0 to 9), containing 60000 training images and 10000 test images. 

\begin{figure*}[tbhp]
\centering
\subfloat[]{\label{fig:f1a}\includegraphics[width=0.45\textwidth]{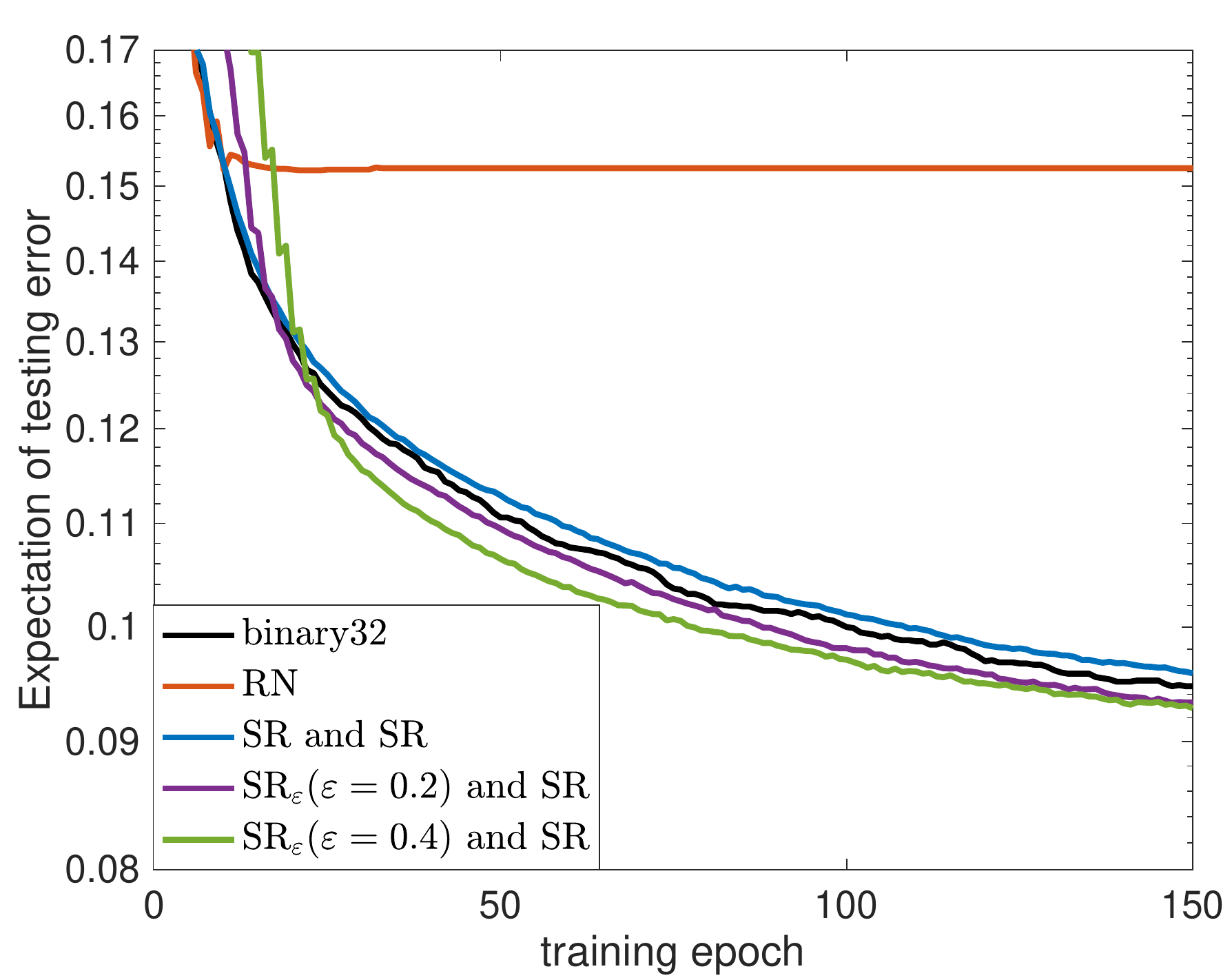}}\quad\quad
\subfloat[]{\label{fig:f1b}\includegraphics[width=0.45\textwidth]{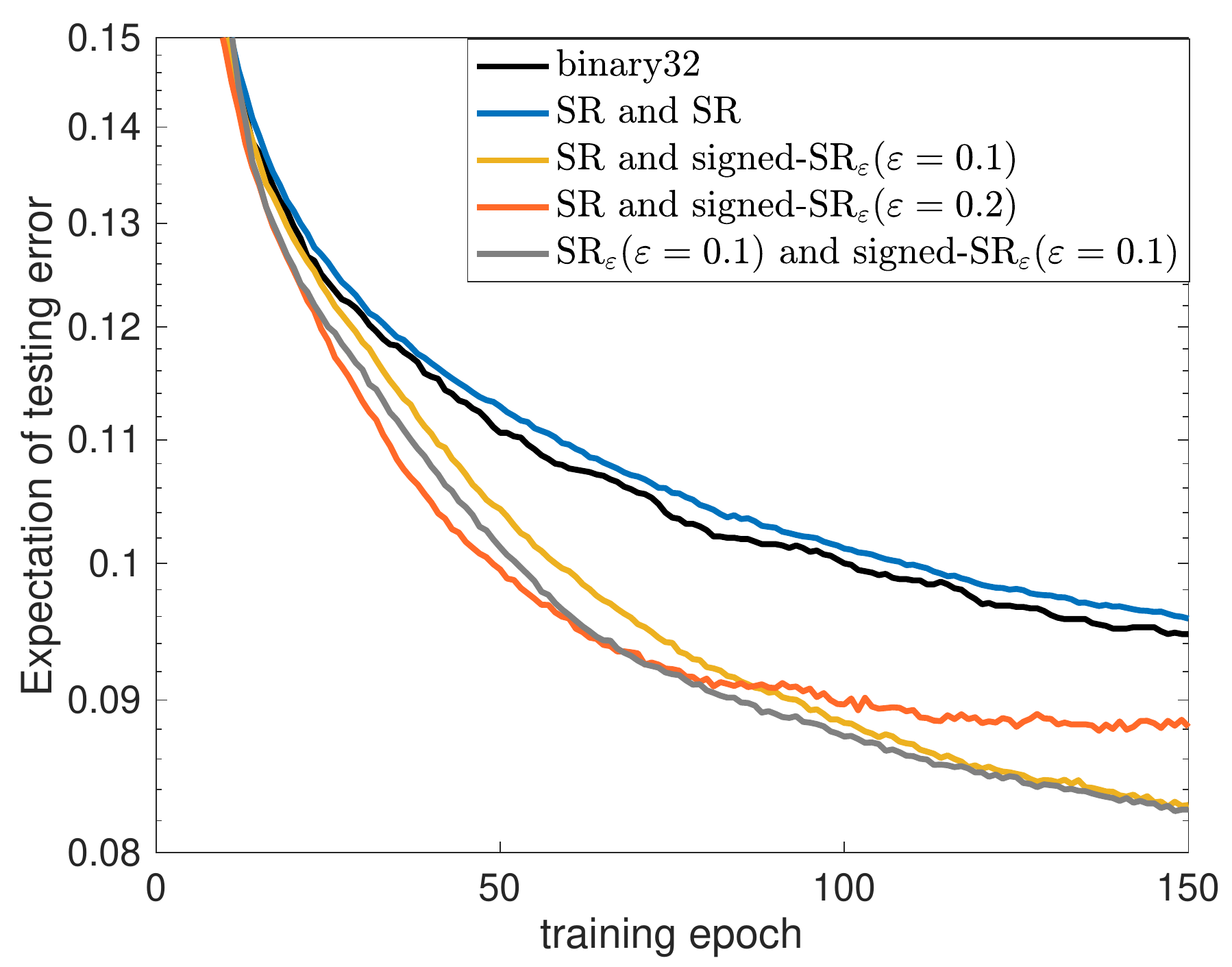}}
\caption{Comparison of the expectation of testing errors of the MLR model with stepsize $t=0.5$ while using $\csr$ to implement \cref{eq:gd_floatp2} and different rounding schemes to implement both \cref{eq:gd_floatp0,eq:gd_floatp1} (\cref{eq:gd_floatp0,eq:gd_floatp1} are implemented using the same rounding scheme), i.e., $\rn$, $\csr$, and $\srh$ with $\varepsilon=0.2, 0.4$ (a) and different combinations of rounding schemes to implement \cref{eq:gd_floatp} (b).}\label{fig:figure1}
\end{figure*}
In our first experiment, we apply $\csr$ to evaluate \cref{eq:gd_floatp2}, and we test different stochastic rounding methods for \cref{eq:gd_floatp1}. \cref{fig:f1a} shows the expectation of testing errors of the MLR model when classifying 0 to 9 with the 10000 test images. After $10$ epochs, {\sf binary8} with $\rn$ stagnates due to the loss of gradient information. With the same number of training epochs, the testing errors of the MLR model obtained by $\csr$ are slightly higher than the baseline, while those achieved by $\srh$ are slightly lower than the baseline, which is consistent with the conclusions after \cref{thm:conv-exp_csr,coro:conv-exp_srh}. Further, a faster convergence is achieved with larger $\varepsilon$ when using $\srh$.
 
In the second experiment, we use $\csr$ and $\srh$ to implement \cref{eq:gd_floatp0,eq:gd_floatp1}, respectively; for \cref{eq:gd_floatp2} we use $\csr$ and $\ssrh$ with the same settings of \cref{lem:expecd_srhsigned}. \cref{fig:f1b} shows the comparison of the expectation of testing errors of the MLR model when implementing GD with different combinations of rounding schemes. It can be seen that the convergence is significantly faster when using $\ssrh$ for \cref{eq:gd_floatp2}. Specifically, with 150 training epochs, the testing error of the baseline is $0.086$. A similar accuracy is obtained by $\ssrh(\varepsilon=0.1)$ with $t=0.1$ and $82$ training epochs. Further increasing the parameter $\varepsilon$ used in $\ssrh$, leads GD to ``jump over" the optimum, which can be seen as employing a very large learning stepsize with exact computations. We also measure the population variance \cite{singh1988estimation} over 20 simulations for all the experiments in \cref{fig:figure1}; after 50 training epochs, all the population variances are less than $10^{-5}$. This indicates small deviations from the average cases.
 
\begin{figure*}[htb!]
\centering
\subfloat[]{\label{fig:lra}\includegraphics[width=0.45\textwidth]{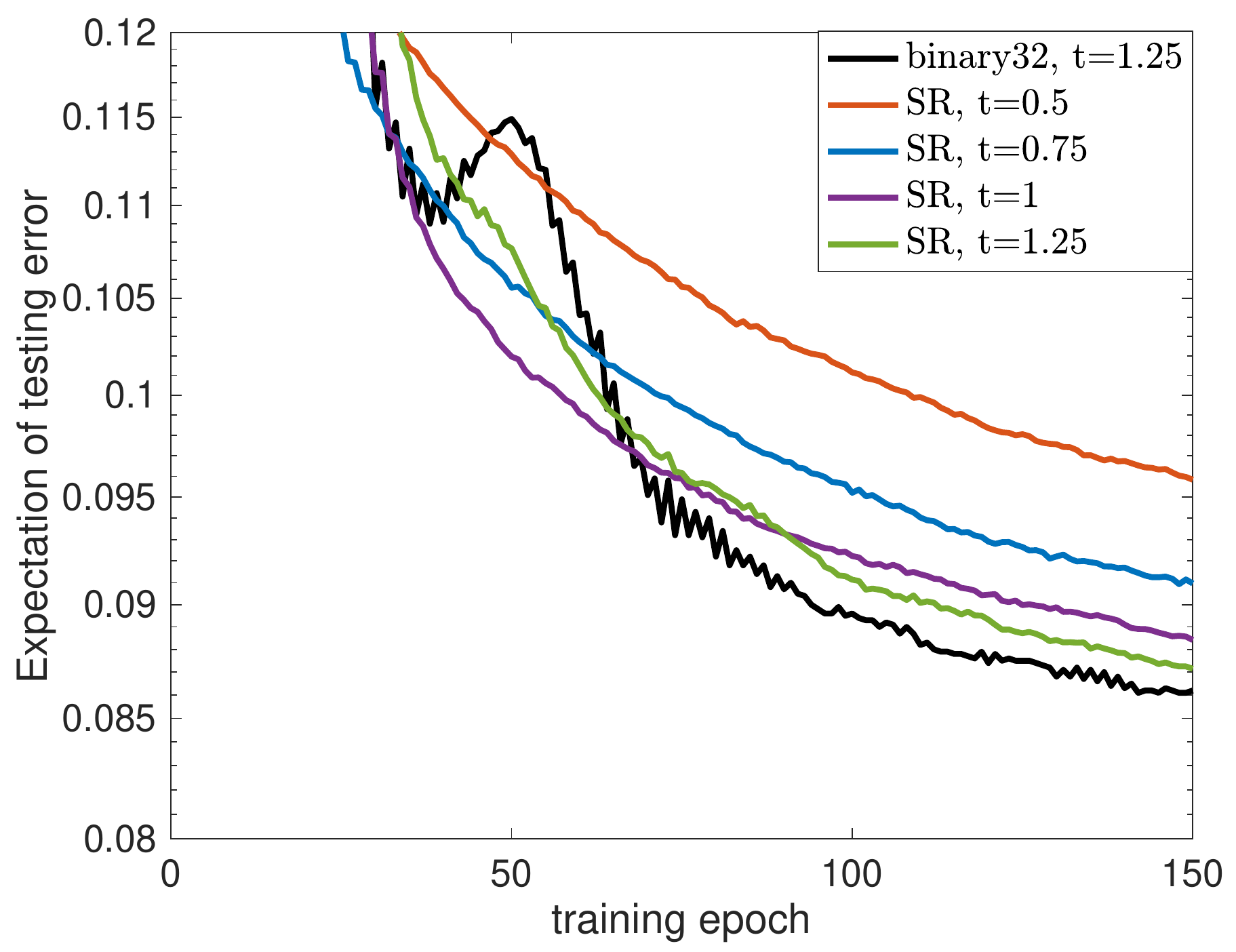}}\quad\quad
\subfloat[]{\label{fig:lrb}\includegraphics[width=0.45\textwidth]{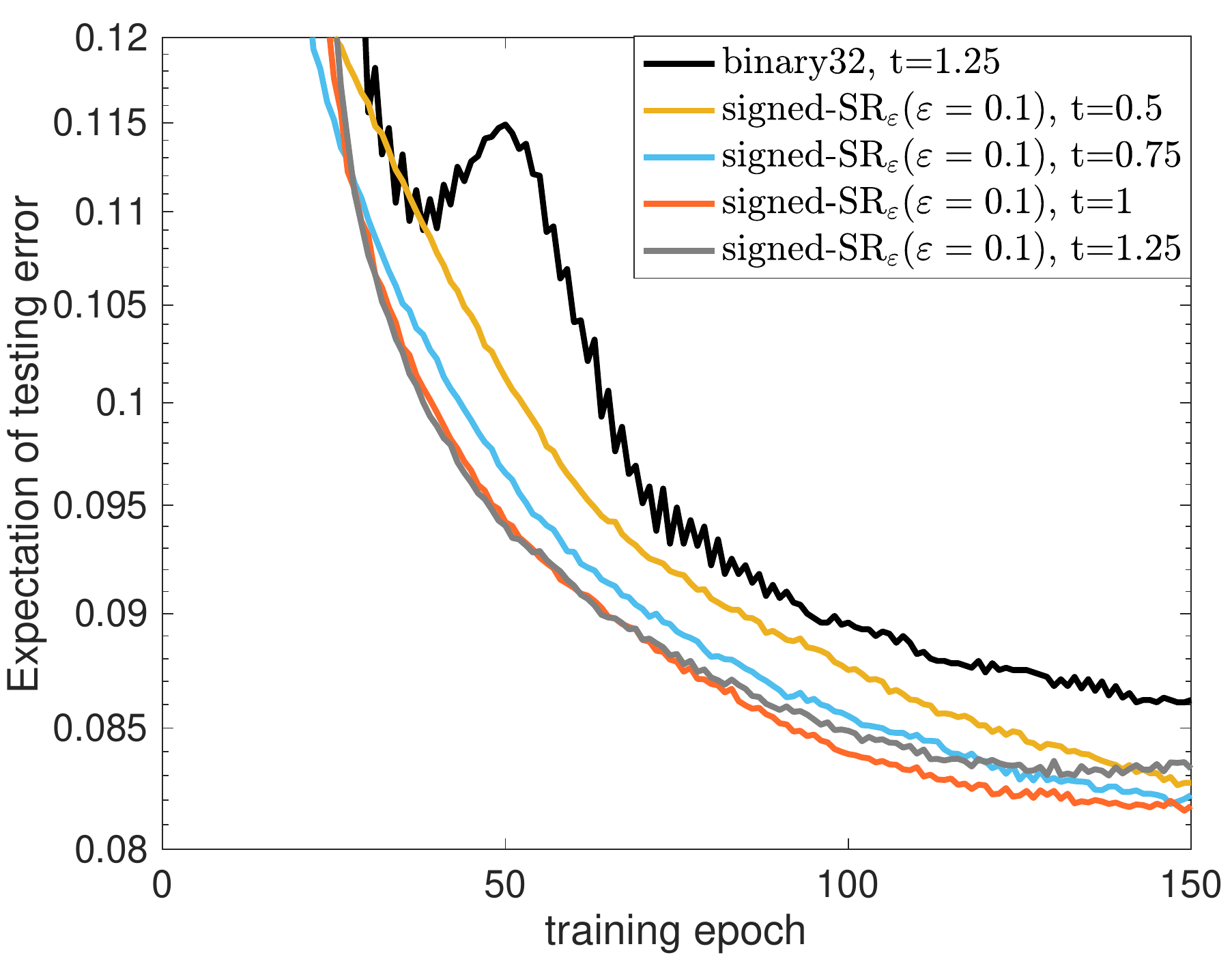}}
\caption{Comparison of the expectation of testing errors of the MLR model with different learning rate $t$ while using $\csr$ to implement \cref{eq:gd_floatp} (a) and using $\srh$ with $\varepsilon=0.1$ to implement \cref{eq:gd_floatp0} and $\ssrh$ with $\varepsilon=0.1$ to implement \cref{,eq:gd_floatp1,eq:gd_floatp2} (b).}\label{fig:learningrate}
\end{figure*}
To further investigate the performances of the various rounding schemes, we analyze the effect of varying the parameter $t$. In \cref{fig:lra} we report the expectation of testing errors of the MLR model with different learning rate $t$ while using $\csr$ to implement both steps of \cref{eq:gd_floatp}. It can be seen that the convergence rate increases with the learning rate $t$ although it never beats the baseline obtained by {\sf binary32} and $t=1.25$. We remark that further increasing $t$, with {\sf binary32}, leads to large oscillations. The experiment is repeated by applying $\ssrh$ with $\varepsilon=0.1$ for \cref{eq:gd_floatp1} and \cref{eq:gd_floatp2}. The results reported in \cref{fig:lrb} show that the convergence obtained with $t=0.5$ is already faster than the baseline. Increasing $t$ until $1$ leads to even faster convergence. However, when $t=1.25$, the testing error starts to increase after 125 training epochs, which indicates that $t=1.25$ is too large for this rounding strategy. With 150 training epochs, the baseline obtains a testing error of $0.086$, while a similar value is obtained by $\ssrh$ with $t=1$ after only 84 training epochs (see \cref{fig:lrb}). 
\subsection{A two-layer NN for binary classification}
Although the training of a two-layer NN is not a convex problem, GD with $\srh$ still shows a similar convergence behavior to the one described when training an MLR model. The training is performed on the images comprised of the digits $3$ and $8$, i.e., $11982$ training images and $1984$ testing images. As in \cite{gupta2015deep}, the pixel values are normalized to $[0,1]$. A two-layer NN is built with the ReLU activation function in the hidden layer and the sigmoid activation function in the output layer. The hidden layer contains 100 units. In the backward propagation, a binary cross-entropy loss function is optimized using GD. The weights matrix is initialized based on Xavier initialization \cite{glorot2010understanding} and the
bias is initialized as a zero vector. Further, the default decision threshold for interpreting probabilities to class labels is $0.5$, since the sample class sizes are almost equal \cite{chen2006decision}. Specifically, class 1 is defined for those predicted scores larger than or equal to $0.5$.

\begin{figure}[htb!]
\centering
\subfloat[]{\label{fig:f2a}\includegraphics[width=0.44\textwidth]{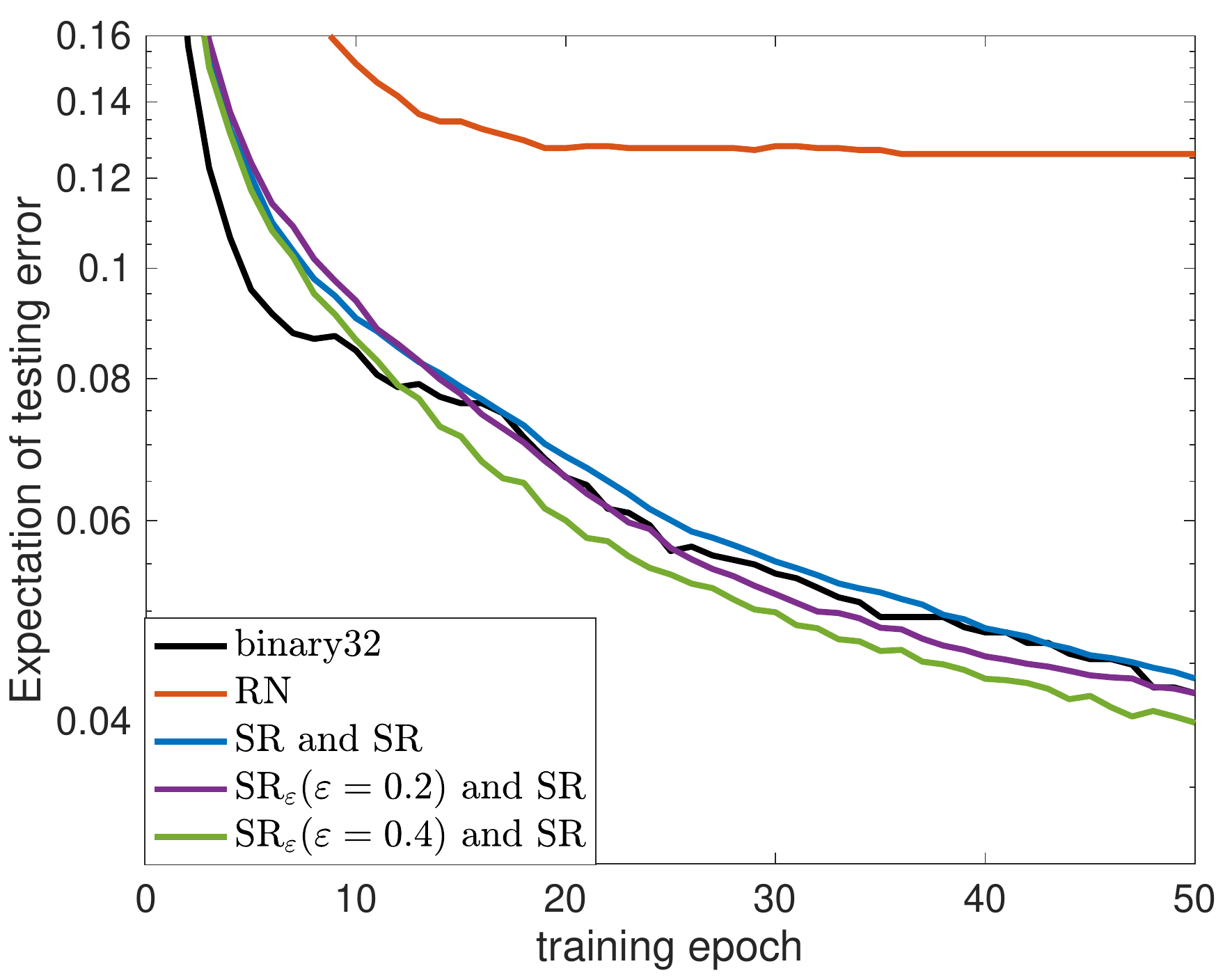}}\quad\quad
\subfloat[]{\label{fig:f2b}\includegraphics[width=0.44\textwidth]{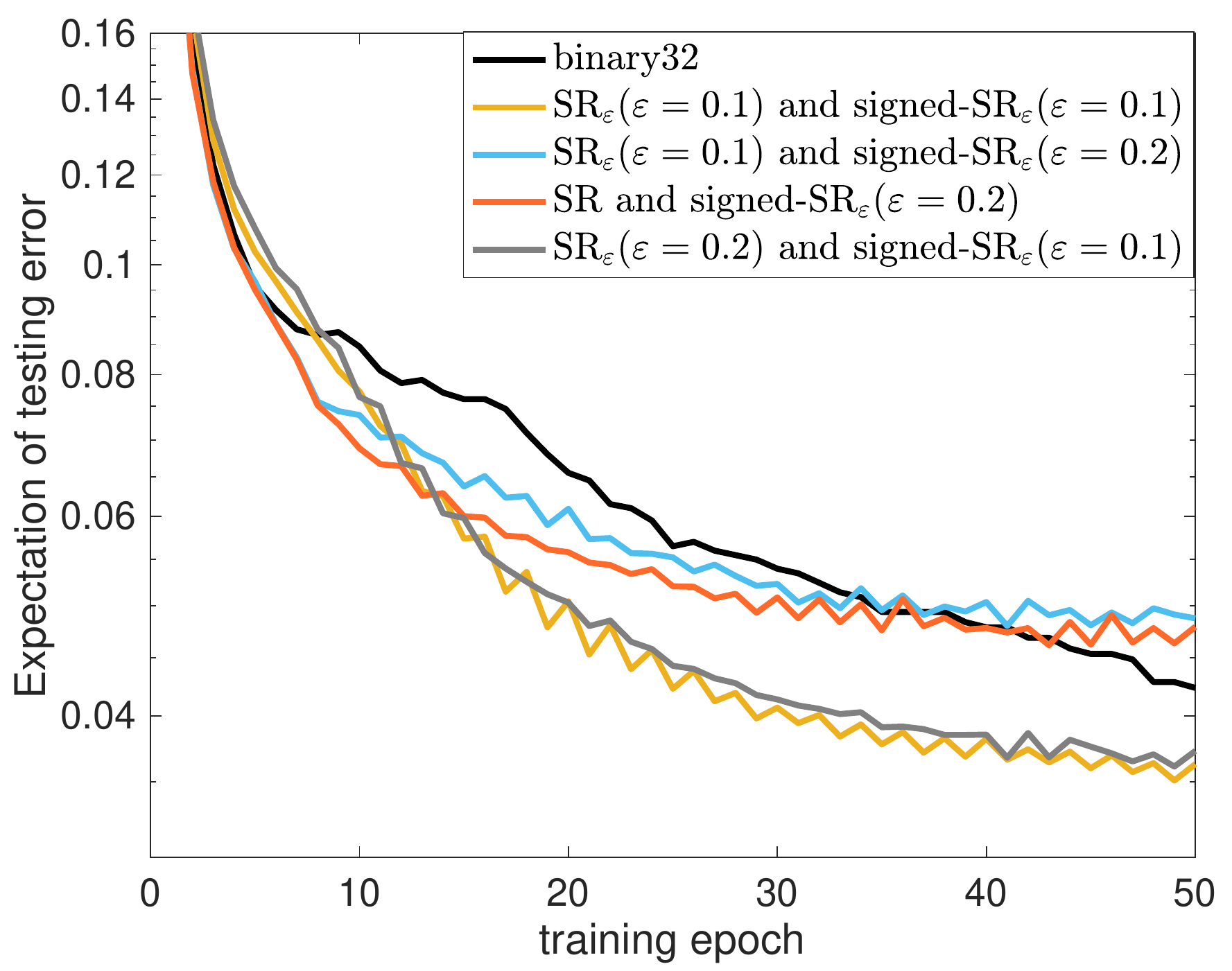}}
\caption{Comparison of the expectation of testing errors of a two-layer NN with $t=0.09375$ when using $\rn$ to implement \cref{eq:gd_floatp}, $\csr$ to implement \cref{eq:gd_floatp2} and different rounding schemes to implement \cref{eq:gd_floatp0,eq:gd_floatp1}, i.e., $\csr$ and $\srh$ with $\varepsilon=0.2, 0.4$ (a) and different combinations of rounding schemes to implement \cref{eq:gd_floatp} (b).} \label{fig:figure2}
\end{figure}

\cref{fig:f2a} shows the comparison of the expectation of testing errors of the two-layer NN trained using {\sf binary8} with $\rn$ for \cref{eq:gd_floatp} and using $\csr$ for \cref{eq:gd_floatp2} and different stochastic rounding methods for \cref{eq:gd_floatp0,eq:gd_floatp1}. Again, the NN trained using $\rn$ fails to converge due to the loss of gradient information. $\csr$ leads to similar testing errors to the baseline, while $\srh$ results in a slightly higher convergence rate than $\csr$. Based on \cref{def:stochasticrounding}, a larger $\varepsilon$ leads to a larger rounding bias, which also leads to slightly faster convergence in \cref{fig:f2a}. To study the influence of rounding bias in each step of \cref{eq:gd_floatp}, we employ $\ssrh$ with the same settings of \cref{lem:expecd_srhsigned} for evaluating \cref{eq:gd_floatp2}. \cref{fig:f2b} shows the expectation of testing errors when implementing GD with different combinations of rounding schemes. Again, the use of $\ssrh$ for \cref{eq:gd_floatp2} yields lower testing errors with less training epochs. For instance, the testing error after 50 training epochs with {\sf binary32} is 0.042, while a similar testing error is obtained after only 25 training epochs when using the combination of $\srh$ and $\ssrh$ (see \cref{fig:f2b}). Also here, a large rounding bias in evaluating the second step \cref{eq:gd_underflowSRB} leads GD to ``jump over" the optimum (e.g., the case with $\varepsilon=0.2$ in \cref{fig:f2b}).

As observed in these numerical studies, the magnitude of the parameter $\varepsilon$ plays a crucial role when implementing $\srh$ or $\ssrh$, as it controls the rounding bias in the descent direction. In particular, $\varepsilon>0$ may accelerate the convergence of GD, but a too large value may also make GD ``jump over'' the optimum. Although, \cref{lem:expecd_srhsigned} indicates that $\varepsilon$ should be less than $0.5$ to guarantee a descent updating direction of GD. By means of numerical studies, we found that the choice of $\varepsilon$ should take into account the machine precision $u$. In the case of {\sf binary8}, we suggest to choose an $\varepsilon\leq0.1$. For most of the numerical studies, the use of $\ssrh$ yields a convergence that is approximately twice as fast as that of $\csr$.
\section{Conclusion} \label{sec:conclusion}
We have studied the influence of rounding bias on the convergence of the gradient descent method (GD) with low-precision floating-point computation for convex problems. We have demonstrated that the use of the unbiased stochastic rounding method ($\csr$) in low-precision computation may only achieve a convergence rate of GD that is close to (slower than) the one attained in exact arithmetic. We have proven that the employment of the two newly proposed rounding methods, $\srh$ and $\ssrh$, may lead to faster convergence of GD than that achieved by $\csr$. The magnitude of the parameter $\varepsilon$ plays a crucial role when implementing $\srh$ or $\ssrh$, since it determines the amount of rounding bias in the descent direction. In particular, we have proven that $\varepsilon$ should be less than $0.5$ to guarantee a descent updating direction of GD. By means of numerical experiments, we have shown that the machine precision $u$ should be considered when selecting the value of $\varepsilon$. In the cases of training a multinomial logistic regression model and a two-layer neural network (NN) with the number format {\sf binary8}, we suggest to choose an $\varepsilon\leq0.1$. In most of the numerical studies, the application of $\ssrh$ produces a convergence rate that is nearly twice as fast as that of $\csr$. The proposed rounding methods may be especially beneficial for machine learning, e.g., training NNs and regression models, where low-precision computations and GD are widely applied. 
\section*{Acknowledgement}
This research was funded by the EU ECSEL Joint Undertaking under grant agreement no.~826452.

\bibliographystyle{ref}
\bibliography{references}

\begin{thebibliography}{10}
\providecommand{\url}[1]{{#1}}
\providecommand{\urlprefix}{URL }
\expandafter\ifx\csname urlstyle\endcsname\relax
  \providecommand{\doi}[1]{DOI~\discretionary{}{}{}#1}\else
  \providecommand{\doi}{DOI~\discretionary{}{}{}\begingroup
  \urlstyle{rm}\Url}\fi

\bibitem{bertsekas2000gradient}
Bertsekas, D.P., Tsitsiklis, J.N.: Gradient convergence in gradient methods
  with errors.
\newblock SIAM J. Optim. \textbf{10}(3), 627--642 (2000)

\bibitem{bohning1992multinomial}
B{\"o}hning, D.: Multinomial logistic regression algorithm.
\newblock Ann. Inst. Stat. Math. \textbf{44}(1), 197--200 (1992)

\bibitem{chen2006decision}
Chen, J., et~al.: Decision threshold adjustment in class prediction.
\newblock SAR QSAR Environ. Res. \textbf{17}(3), 337--352 (2006)

\bibitem{chung2018serving}
Chung, E., et~al.: Serving {DNNs} in real time at datacenter scale with project
  brainwave.
\newblock IEEE Micro \textbf{38}(2), 8--20 (2018)

\bibitem{connolly2021stochastic}
Connolly, M.P., Higham, N.J., Mary, T.: Stochastic rounding and its
  probabilistic backward error analysis.
\newblock SIAM J. Sci. Comput. \textbf{43}(1), A566--A585 (2021)

\bibitem{croci2022stochastic}
Croci, M., Fasi, M., Higham, N.J., Mary, T., Mikaitis, M.: Stochastic rounding:
  implementation, error analysis and applications.
\newblock R. Soc. Open Sci. \textbf{9}(3), 211,631 (2022)

\bibitem{croci2020effects}
Croci, M., Giles, M.B.: Effects of round-to-nearest and stochastic rounding in
  the numerical solution of the heat equation in low precision.
\newblock IMA J. Numer. Anal.  (2022)

\bibitem{davies2018loihi}
Davies, M., et~al.: Loihi: A neuromorphic manycore processor with on-chip
  learning.
\newblock IEEE Micro \textbf{38}(1), 82--99 (2018)

\bibitem{deng2012mnist}
Deng, L.: The {MNIST} database of handwritten digit images for machine learning
  research.
\newblock IEEE Signal Process. Mag. \textbf{29}(6), 141--142 (2012)

\bibitem{glorot2010understanding}
Glorot, X., Bengio, Y.: Understanding the difficulty of training deep
  feedforward neural networks.
\newblock In: Proc. of the 13th Int. Conf. Artif. Intell. Stat., pp. 249--256
  (2010)

\bibitem{gupta2015deep}
Gupta, S., Agrawal, A., Gopalakrishnan, K., Narayanan, P.: Deep learning with
  limited numerical precision.
\newblock In: Proc. of the 32nd Int. Conf. Mach. Learn., pp. 1737--1746 (2015)

\bibitem{hickmann2020intel}
Hickmann, B., et~al.: {Intel Nervana} neural network processor-t ({NNP-T})
  fused floating point many-term dot product.
\newblock In: Proc. of the 27th IEEE Symp. Comput., pp. 133--136. IEEE (2020)

\bibitem{higham2002accuracy}
Higham, N.J.: {Accuracy and Stability of Numerical Algorithms}.
\newblock SIAM (2002)

\bibitem{higham2019simulating}
Higham, N.J., Pranesh, S.: Simulating low precision floating-point arithmetic.
\newblock SIAM J. Sci. Comput. \textbf{41}(5), C585--C602 (2019)

\bibitem{hopkins2019stochastic}
Hopkins, M., Mikaitis, M., Lester, D.R., Furber, S.: Stochastic rounding and
  reduced-precision fixed-point arithmetic for solving neural ordinary
  differential equations.
\newblock Philos. Trans. Royal Soc. A \textbf{378}(2166), 20190,052 (2020)

\bibitem{hosmer2013applied}
Hosmer~Jr, D.W., Lemeshow, S., Sturdivant, R.X.: Applied Logistic Regression.
\newblock John Wiley \& Sons (2013)

\bibitem{huskey1949precision}
Huskey, H.D., Hartree, D.R.: On the precision of a certain procedure of
  numerical integration.
\newblock J. Res. Natl. Inst. Stand. Technol. \textbf{42}, 57--62 (1949)

\bibitem{8766229}
IEEE: {IEEE} standard for floating-point arithmetic.
\newblock IEEE Std 754-2019 (Revision of {IEEE} 754-2008) pp. 1--84 (2019)

\bibitem{jouppi2020domain}
Jouppi, N.P., et~al.: A domain-specific supercomputer for training deep neural
  networks.
\newblock Commun. ACM \textbf{63}(7), 67--78 (2020)

\bibitem{kuczma2009introduction}
Kuczma, M.: An {I}ntroduction to the {T}heory of {F}unctional {E}quations and
  {I}nequalities: {Cauchy}'s {E}quation and {Jensen's} {I}nequality.
\newblock Springer Science \& Business Media (2009)

\bibitem{lee2016gradient}
Lee, J.D., Simchowitz, M., Jordan, M.I., Recht, B.: Gradient descent only
  converges to minimizers.
\newblock In: Proc. of the 29th Annual Conf. on Learn. Theory, pp. 1246--1257.
  PMLR (2016)

\bibitem{li2017training}
Li, H., et~al.: Training quantized nets: A deeper understanding.
\newblock In: Proc. of the 31st Neural Inf. Process. Syst. Conf., vol.~30
  (2017)

\bibitem{liu2015fuzzy}
Liu, Y., Gao, Y., Tong, S., Li, Y.: Fuzzy approximation-based adaptive
  backstepping optimal control for a class of nonlinear discrete-time systems
  with dead-zone.
\newblock IEEE Trans. Fuzzy Syst. \textbf{24}(1), 16--28 (2015)

\bibitem{mikaitis2021stochastic}
Mikaitis, M.: Stochastic rounding: Algorithms and hardware accelerator.
\newblock In: Proc. of 2021 Int. Jt. Conf. Neural Netw., pp. 1--6. IEEE (2021)

\bibitem{moulay2019properties}
Moulay, E., L{\'e}chapp{\'e}, V., Plestan, F.: Properties of the sign gradient
  descent algorithms.
\newblock Inf. Sci. \textbf{492}, 29--39 (2019)

\bibitem{na2017chip}
Na, T., Ko, J.H., Kung, J., Mukhopadhyay, S.: On-chip training of recurrent
  neural networks with limited numerical precision.
\newblock In: Proc. of the 2017 Int. Jt. Conf. Neural Netw., pp. 3716--3723.
  IEEE (2017)

\bibitem{nesterov2003introductory}
Nesterov, Y.: Introductory Lectures on Convex Optimization: A Basic Course.
\newblock Springer (2003)

\bibitem{navidiah100}
{NVIDIA H100} tensor core {GPU} architecture [white paper] (2022)

\bibitem{ortiz2018low}
Ortiz, M., Cristal, A., Ayguad{\'e}, E., Casas, M.: Low-precision
  floating-point schemes for neural network training.
\newblock arXiv preprint: 1804.05267  (2018)

\bibitem{paxton2021climate}
Paxton, E.A., et~al.: Climate modeling in low precision: Effects of both
  deterministic and stochastic rounding.
\newblock J. Clim. \textbf{35}(4), 1215--1229 (2022)

\bibitem{petres2007path}
Petres, C., et~al.: Path planning for autonomous underwater vehicles.
\newblock IEEE Trans. Robot. \textbf{23}(2), 331--341 (2007)

\bibitem{schmidt2011convergence}
Schmidt, M., Roux, N., Bach, F.: Convergence rates of inexact proximal-gradient
  methods for convex optimization.
\newblock In: Proc. of the 24th Neural Inf. Process. Syst. Conf., pp.
  1458--1466 (2011)

\bibitem{singh1988estimation}
Singh, H., Upadhyaya, L., Namjoshi, U.: Estimation of finite population
  variance.
\newblock Curr. Sci pp. 1331--1334 (1988)

\bibitem{steyer2017probability}
Steyer, R., Nagel, W.: Probability and Conditional Expectation: Fundamentals
  for the Empirical Sciences.
\newblock John Wiley \& Sons (2017)

\bibitem{su2020towards}
Su, C., Zhou, S., Feng, L., Zhang, W.: Towards high performance low bitwidth
  training for deep neural networks.
\newblock J. Semicond. \textbf{41}(2), 022,404 (2020)

\bibitem{wang2018training}
Wang, N., et~al.: Training deep neural networks with 8-bit floating point
  numbers.
\newblock In: Proc. of the 31st Neural Inf. Process. Syst. Conf., pp.
  7675--7684 (2018)

\bibitem{zou2020gradient}
Zou, D., Cao, Y., Zhou, D., Gu, Q.: Gradient descent optimizes
  over-parameterized deep {ReLU} networks.
\newblock Mach. Learn. \textbf{109}(3), 467--492 (2020)

\end{thebibliography}

\end{document}